\documentclass{article}

\usepackage[utf8]{inputenc} 
\usepackage[T1]{fontenc}    
\usepackage{xr} 
\usepackage{hyperref}       
\usepackage{url}            
\usepackage{booktabs}       
\usepackage{amsfonts}       
\usepackage{nicefrac}       
\usepackage{microtype}      
\usepackage{amsthm}
\usepackage{comment}
\usepackage{mathtools}
\usepackage[algo2e]{algorithm2e}
\usepackage{graphicx}
\usepackage{subcaption}
\usepackage{enumitem}

\usepackage{xcolor}
\usepackage{amsmath}
\usepackage{verbatim}
\newcommand{\WS}{WS}
\newcommand{\WSWEAK}{WS-W}
\newcommand{\WSSTRONG}{WS-S}
\newcommand\floor[1]{\lfloor#1\rfloor}

\usepackage{dblfloatfix}
\usepackage{verbatim}
\usepackage{achicago}

\theoremstyle{definition}

\newtheorem{theorem}{Theorem}
\newtheorem{lemma}{Lemma}

\usepackage{times}

\usepackage{natbib}

\usepackage{hyperref}

\usepackage[accepted]{icml2017}

\icmltitlerunning{Submission and Formatting Instructions for ICML 2017}

\begin{document} 

\twocolumn[
\icmltitle{Dueling Bandits with Weak Regret}




\begin{icmlauthorlist}
\icmlauthor{Bangrui Chen}{to}
\icmlauthor{Peter I. Frazier}{to}
\end{icmlauthorlist}

\icmlaffiliation{to}{Cornell University, Ithaca, NY}

\icmlcorrespondingauthor{Peter I. Frazier}{pf98@cornell.edu}

\icmlkeywords{boring formatting information, machine learning, ICML}

\vskip 0.3in
]



\printAffiliationsAndNotice{} 

\begin{abstract} 
We consider online content recommendation with implicit feedback through pairwise comparisons, formalized as the so-called dueling bandit problem.  We study the dueling bandit problem in the Condorcet winner setting, and consider two notions of regret: the more well-studied strong regret, which is 0 only when both arms pulled are the Condorcet winner; and the less well-studied weak regret, which is 0 if either arm pulled is the Condorcet winner. 
We propose a new algorithm for this problem, {\it Winner Stays} (\WS), with variations for each kind of regret:
\WS\ for weak regret (\WSWEAK) has expected cumulative weak regret that is $O(N^2)$, and $O(N\log(N))$ if arms have a total order; WS for strong regret (\WSSTRONG) has expected cumulative strong regret of $O(N^2 + N \log(T))$, and $O(N\log(N)+N\log(T))$ if arms have a total order.
\WSWEAK\ is the first dueling bandit algorithm with weak regret that is constant in time.
\WS\ is simple to compute, even for problems with many arms, and we demonstrate through numerical experiments on simulated and real data that \WS\ has significantly smaller regret than existing algorithms in both the weak- and strong-regret settings.
\end{abstract} 

\section{Introduction}

We consider bandit learning in personalized content recommendation with implicit pairwise comparisons. We offer pairs of items to a user and record implicit feedback on which offered item is preferred, seeking to learn the user's preferences over items quickly, while also ensuring that the fraction of time we fail to offer a high-quality item is small.  Implicit pairwise comparisons avoid the inaccuracy of user ratings \cite{joachims2007evaluating} and the difficulty of engaging users in providing explicit feedback. 

We study a model for this setting called the dueling bandit problem \cite{yue2009interactively}.  The items we may offer to the user are called ``arms'', and we learn about these arms through a sequence of ``duels''.  In each duel, we ``pull'' two arms and receive noisy feedback from the user telling us which arm is preferred.
When an arm is preferred within a duel, we say that the arm has ``won the duel".


We study this problem in the Condorcet winner setting, in which we assume the existence of an arm (the Condorcet winner) that wins with probability at least $\frac{1}{2}$ when paired with any of the other arms. 
In these settings, we consider two notions of regret: ``weak regret``, in which we avoid regret by selecting the Condorcet winner as either arm in the duel; and ``strong-regret'', in which we can only avoid regret by setting both arms in the duel to the Condorcet winner. 

Weak regret was proposed by \citet{yue2012k} and arises in content recommendation when arms correspond to items, and the user incurs no regret whenever his most preferred item is made available.  Examples include in-app restaurant recommendations provided by food delivery services like Grubhub and UberEATS, in which implicit feedback may be inferred from selections, and the user only incurs regret if her most preferred restaurant is not recommended.  Examples also include recommendation of online broadcasters on platforms such as Twitch, in which implicit feedback may again be inferred from selections, and  the user is fully satisfied as long as her favored broadcaster is listed.
Despite its applicability, \citet{yue2012k} is the only paper of which we are aware that studies weak regret, and it does not provide algorithms specifically designed for this setting.

Strong regret has been more widely studied, as discussed below, and has application to choosing ranking algorithms for search \cite{hofmann2013fidelity}.  To perform a duel, query results from two rankers are interleaved \cite{radlinski2008does}, and the ranking algorithm that provided the first result chosen by the user is declared the winner of the duel.  Strong regret is appropriate in this setting because the user's experience is enhanced by pulling the best arm twice, so that all of that ranker's results are shown.

Our contribution is a new algorithm, {\it Winner Stays} (\WS), with variants designed for the weak (\WSWEAK) and strong regret (\WSSTRONG) settings. We prove that \WSWEAK\ has expected cumulative weak regret that is constant in time, with dependence on the number of arms $N$ given by $O(N^2)$.  If the arms have a total order, we show a tighter bound of $O(N \log N)$.  We then prove that \WSSTRONG\ has expected cumulative strong regret that is $O(N^2+N\log(T))$, and prove that a tighter bound of $O(N \log(N) + N\log(T))$ holds if arms have a total order.  
These regret bounds are optimal in $T$, and for weak regret are strictly better than those for any previously proposed algorithm, although at the same time both strong and weak regret bounds are sensitive to the minimum gap in winning probability between arms. 
We demonstrate through numerical experiments on simulated and real data that \WSWEAK\ and \WSSTRONG\ significantly outperform existing algorithms on strong and weak regret.

The paper is structured as follows. 
Section~\ref{relatedWork} reviews related work.
Section~\ref{probForm} formulates our problem. 
Section~\ref{Methods} introduces the \textit{Winner Stays} (\WS) algorithm: Section~\ref{WS1} defines \WSWEAK\ for the weak regret setting; Section~\ref{WS1Proof} proves that \WSWEAK\ has cumulative expected regret that is constant in time; Section~\ref{WS2} defines \WSSTRONG\ for the strong regret setting and bounds its regret. Section~\ref{utility} disusses a simple extension of our theoretical results to the utility-based bandit setting, which is used in our numerical experiments.  Section~\ref{Numerical} compares \WS\ with three benchmark algorithms using both simulated and real datasets, finding that \WS\ outperforms these benchmarks on the problems considered.

\section{Related Work} \label{relatedWork}

Most work on dueling bandits focuses on strong regret. \citet{yue2012k} shows that the worst-case expected cumulative strong regret up to time T for any algorithm is $\Omega(N\log(T))$. 
Algorithms have been proposed that reach this lower bound under the Condorcet winner assumption in the finite-horizon setting: Interleaved Filter (IF) \cite{yue2012k} and Beat the Mean (BTM) \cite{yue2011beat}.  Relative Upper Confidence Bound (RUCB) \cite{zoghi2014relative} also reaches this lower bound in the horizonless setting. Relative Minimum Empirical Divergence (RMED) \cite{komiyama2015regret} is the first algorithm to have a regret bound that matches this lower bound. \citet{zoghi2015copeland} proposed two algorithms, Copeland Confidence Bound (CCB) and Scalable Copeland Bandits (SCB), which achieve an optimal regret bound without assuming existence of a Condorcet winner. 

While weak regret was proposed in \citet{yue2012k}, it has not been widely studied to our knowledge, and despite its applicability we are unaware of papers that provide algorithms designed for it specifically.  
While one can apply algorithms designed for the strong regret setting to weak regret, and use the fact that strong dominates weak regret to obtain weak regret bounds of $O(N\log(T))$, these are looser than the constant-in-$T$ bounds that we show.

Active learning using pairwise comparisons is also closely related to our work. \citet{jamieson2011active} considers an active learning problem that is similar to our problem in that the primary goal is to sort arms based on the user's preferences, using adaptive pairwise comparisons. It proposes a novel algorithm, the Query Selection Algorithm (QSA), that uses an expected number of operations of $d\log(N)$ to sort $N$ arms, where $d$ is the dimension of the space in which the arms are embedded, rather than $N\log(N)$.  \citet{busa2013top} and \citet{busa2014preference} consider top-k element selection using adaptive pairwise comparisons. They propose a generalized racing algorithm focusing on minimizing sample complexity.  \cite{pallone2017} studies adaptive preference learning across arms using pairwise preferences.  They show that a greedy algorithm is Bayes-optimal for an entropy objective.  While similar in that they use pairwise comparisons, these algorithms are different in focus from the current work because they do not consider cumulative regret.

\section{Problem Formulation}
\label{probForm}

We consider $N$ items (arms). At each time $t=1,2,\ldots$, the system chooses two items and shows them to the user, i.e., the system performs a duel between two arms. The user then provides binary feedback indicating her preferred item, determining which arm wins the duel. This binary feedback is random, and is conditionally independent of all past interactions given the pair of arms shown. We let $p_{i,j}$ denote the probability that the user gives feedback indicating a preference for arm $i$, when shown arms $i$ and $j$. If the user prefers arm $i$ over arm $j$, we assume $p_{i,j}>0.5$. We also assume symmetry: $p_{i,j} = 1-p_{j,i}$. 

We assume arm $1$ is a Condorcet winner, i.e., that $p_{1,i}>0.5$ for $i=2,\cdots, N$. In some results, we also consider the setting in which arms have a total order, by which we mean that the arms are ordered so that $p_{i,j}>0.5$ for all $i<j$.  The total order assumption implies transitivity.


We let $p=\min_{p_{i,j}>0.5}p_{i,j}>0.5$ be a lower bound on the probability that the user will choose her favourite arm.

We consider both weak and strong regret in its binary form. The single-period {\it weak regret} incurred at this time is 
$r(t) = 1$ if we do not pull the best arm and $r(t)=0$ otherwise. 
The single-period {\it strong regret} is 
$r(t) = 1$ if we do not pull the best arm twice and $r(t)=0$ otherwise. 
We also consider utility-based extensions of weak and strong regret in Section~\ref{utility}.


We use the same notation $r(t)$ to denote strong and weak regret, and rely on context to distinguish the two cases. In both cases, we define the cumulative regret up to time $T$ to be $R(T)=\sum_{t=1}^{T}r(t)$.
We measure the quality of an algorithm by its expected cumulative regret. 


\section{Winner Stays}
\label{Methods}

We now propose an algorithm, called \textit{Winner Stays} (\WS), with two variants: \WSWEAK\, designed for weak regret; and \WSSTRONG\ for strong regret. Section~\ref{WS1} introduces \WSWEAK\ and illustrates its dynamics. Section~\ref{WS1Proof} proves the expected cumulative weak regret of \WSWEAK\ is $O(N^2)$ under the Condorcet winner setting, and  $O(N\log(N))$ under the total order setting.
Section~\ref{WS2} introduces \WSSTRONG\ and proves that its expected cumulative strong regret is 
$O(N^2+N\log(T))$ under the Condorcet winner setting, and $O(N\log(T) + N\log(N))$ under the total order setting, both of which have optimal dependence on $T$. Section~\ref{utility} extends our theoretical results to utility-based bandits.

\subsection{Winner Stays with Weak Regret (WS-W)}
\label{WS1}


We now present \WSWEAK, first defining some notation. Let $q_{i,j}(t)$ be the number of times that arm $i$ has defeated arm $j$ in a duel, up to and including time $t$.  Then, define $C(t,i)=\sum_{j\neq i} q_{i,j}(t)-q_{j,i}(t)$. 
$C(t,i)$ is the difference between the number of duels won and lost by arm $i$, up to time $t$. 
With this notation, we define \WSWEAK\ in Algorithm~\ref{algoweak}.

\begin{algorithm}
Input: arms $1,\cdots, N$ \\
 \For{$t=1,2,\cdots$}{
  Step 1: Pick $i_{t}\!=\!\arg\max_{i}C(t\!-\!1,i)$, breaking ties as follows:
  \vspace{-3mm}
  \begin{itemize}[leftmargin=0.15in,rightmargin=0.15in]
  \setlength\itemsep{-3pt}
  \item If $t>1$ and $i_{t-1} \in \arg\max_{i}C(t\!-\!1,i)$, set $i_t=i_{t-1}$.
  \item Else if $t\!>\!1$ and $j_{t-1} \!\in\! \arg\max_{i}C(t\!-\!1,i)$,  \break set $i_t=j_{t-1}$.
  \item Else choose $i_t$ uniformly at random from $\arg\max_{i}C(t-1,i)$.
  \end{itemize}
   \vspace{-3mm}
  Step 2: Pick $j_{t}\!=\!\arg\max_{j\!\neq\! i_{t}}C(t\!-\!1,j)$, breaking ties as follows:
  \vspace{-3mm}
  \begin{itemize}[leftmargin=0.15in,rightmargin=0.15in]
  \setlength\itemsep{-3pt}
  \item If $t>1$ and $i_{t-1} \in \arg\max_{i\ne i_t}C(t-1,i)\setminus\{i_t\}$, \break set $j_t=i_{t-1}$.
  \item Else if $t>1$ and $j_{t-1} \in \arg\max_{i\ne i_t}C(t-1,i)\setminus\{i_t\}$, set $j_t=j_{t-1}$.
  \item Else choose $j_t$ uniformly at random from $\arg\max_{j}C(t-1,j)\setminus\{i_{t}\}$.
  \end{itemize}  
  \vspace{-3mm}
  Step 3: Pull arms $i_{t}$ and $j_{t}$;\\
  Step 4: Observe noisy binary feedback and update $C(t,i_{t})$ and $C(t,j_{t})$;
 }
 \caption{\WSWEAK}
 \label{algoweak}
\end{algorithm}

\WSWEAK's pulls can be organized into {\it iterations}, each of which consists of a sequence of pulls of the same pair of arms, and {\it rounds}, each of which consists of a sequence of iterations in which arms that lose an iteration are not visited again until the next round.  We first describe iterations and rounds informally with an example and in Figure~\ref{illustration}
before presenting our formal analysis.

{\bf Example}:
At time $t=1$, $C(0,i)=0$ for all $i$, and \WSWEAK\ pulls two randomly chosen arms.  Suppose it pulls arms $i_1=1$, $j_1=2$ and arm $1$ wins. Then $C(1,i)$ is 1 for arm 1, $-1$ for arm 2, and $0$ for the other arms.  
This first pull is an iteration of length 1, arm $1$ is the winner, and arm $2$ is the loser.  This iteration is in the first round.  We call $t_1 = 1$ the start of the first round, and $t_{1,1} = 1$ the start of the first iteration in the first round.

At time $t=2$, $C(t\!-\!1,i)$ is largest for arm $1$ so \WSWEAK\ chooses $i_2=1$.   Since $C(t-1,i)$ is $-1$ for arm 2 and 0 for the other arms, \WSWEAK\ chooses $j_{2}$ at random from arms 3 through $N$ (suppose $N>2$).  Suppose it chooses arm $j_2=3$.  
This pair of arms (1 and 3) is different from the pair pulled in the previous iteration (1 and 2), so $t_{1,2}=2$ is the start of the second iteration (in the first round).

\WSWEAK\ continues pulling arms 1 and 3 until $C(t,i)$ is $-1$ for one of these arms and $2$ for the other.  \WSWEAK\ continues to pull only arms 1 and 3 until one has $C(t,i)=2$ even though this may involve times when $C(t,i)$ is 0 for both arms 1 and 3, causing them to be tied with arms 4 and above, because we break ties to prioritize pulling previously pulled arms.  The sequence of times when we pull arms 1 and 3 is the second iteration.  The arm that ends the iteration with $C(t,i)=2$ is the winner of that iteration.




\begin{figure*}
\includegraphics[width=\textwidth]{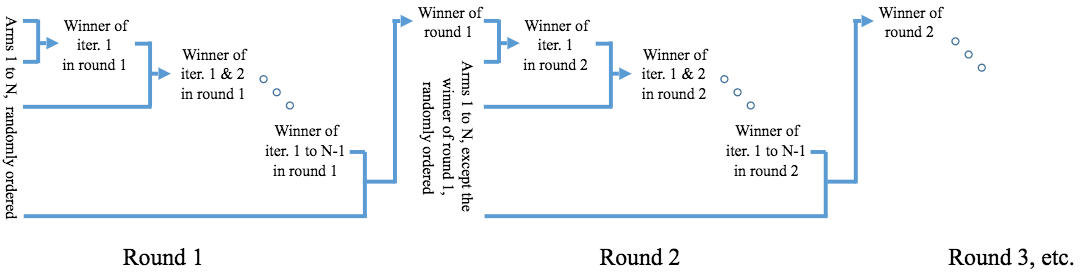}
\caption{Our analysis of \WSWEAK\ decomposes its behavior into a sequence of rounds.  In each round, pairs of arms play each other in a sequence of iterations.  The winner from an iteration passes on to play a new arm in the next iteration randomly selected from those that have not yet played in the round.  At the end of a round, the round's winner is considered first in the next round.}
\label{illustration}
\end{figure*}

\WSWEAK\ continues this process, performing $N-1$ iterations on different pairs of arms, pitting the winner of each iteration against a previously unplayed arm in the next iteration.  
This sequence of iterations is the first round.
The winner of the final iteration in the first round, call it arm $Z(1)$, has $C(t,Z(1))\!=N\!-\!1$ and all other arms $j\!\neq\! Z(1)$ have $C(t,j)\!=\!-1$.

The second round begins on the next pull after the end of the first round, at time $t_2$.  \WSWEAK\ again performs $N-1$ iterations, playing $Z(1)$ in the first iteration.  Each iteration has a winner that passes to the next iteration.  

\WSWEAK\ repeats this process for an infinite number of rounds. Each round is a sequence of $N\!-\!1$ iterations, and an arm that loses an iteration is not revisited until the next round. Figure~\ref{illustration} illustrates these dynamics, and we formalize the definition of round and iteration in the next section.
 
\subsection{Analysis of \WSWEAK}
\label{WS1Proof}

In this section, we analyze the weak regret of \WSWEAK. After presenting definitions and preliminary results, we prove \WSWEAK\ has expected cumulative weak regret bounded by $O(N\log(N))$ when arms have a total order. Then, in the more general Condorcet winner setting, we prove \WSWEAK\ has expected cumulative weak regret bounded by $O(N^2)$.
We leave the proofs of all lemmas to the supplement.

We define $t_{\ell}$, the {\it beginning of round $\ell$}, and $Z(\ell-1)$, the {\it winner} of round $\ell$, as the unique time and arm such $C(t_{\ell}-1, Z(\ell-1))=(N-1)(\ell-1)$ and $C(t_{\ell}-1,i)=-\ell+1$ for all $i\neq Z(\ell-1)$. 

We define $t_{\ell,k}$, the {\it beginning of iteration $k$ in round $\ell$}, as the first time we pull the $k^{th}$ unique pair of arms in the $\ell^{th}$ round.  We let $T_{\ell,k}$ be the number of successive pulls of this pair of arms. 

We additionally define terminology to describe arms pulled in an iteration.
In a duel between arms $i$ and $j$ with $p_{i,j}\!>\!0.5$, arm $i$ is called the \textit{better arm} and arm $j$ is called the \textit{worse arm}.
We say that an arm $i$ is the \textit{incumbent} in iteration $k$ iteration and round $\ell$ if $C(t_{\ell,k}\!-\!1,\!i)\!>\!0$. 
A unique such arm exists except when $\ell=k=1$.  When $\ell=k=1$, the incumbent is the better of the two arms being played.
We call the arm being played that is not the incumbent the  \textit{challenger}.





Using these definitions, we present our first pair of results toward bounding the expected cumulative weak regret of \WSWEAK.  They bound the number of pulls in an iteration.


\begin{lemma}
The conditional expected length of iteration $k$ in round $\ell$,
given the arms being pulled, is bounded above by $\frac{N(\ell-1)+k}{2p-1}$ 
if the incumbent is worse than the challenger, and by $\frac{1}{2p-1}$
if the incumbent is better than the challenger.
\label{adv}
\end{lemma}

Lemma~\ref{adv} shows that iterations with a worse incumbent use more pulls. We then bound the number of iterations with a worse incumbent. 

\begin{lemma}
Under the total order assumption, the conditional expected number of future iterations with an incumbent worse than the challenger, given history up to time $t_{\ell,k}$,
is bounded above by $\frac{2 p^2}{(2p-1)^{3}}(\log(N)+1)$ for any $k,\ell\ge1$.
\label{count}
\end{lemma}

Lemma~\ref{count} implies that the incumbent is worse than the challenger in finitely many iterations with probability $1$.   We now bound the tail distribution of the last such round.

\begin{lemma}
Let $L$ denote the smallest $\ell$ such that no round $\ell' > \ell$ contains an iteration in which the incumbent is worse than the challenger.  Then 
$P(L\geq \ell)\leq \left(\frac{1-p}{p}\right)^{\ell}$.
\label{tail}
\end{lemma}

To present our final set of preliminary lemmas, we define several indicator functions.
Let $B(\ell,k)$ be $1$ when the incumbent in iteration $k$ of round $\ell$ is better than the challenger. 
Let $D(\ell)$ be $1$ if arm 1 (the best arm) is the incumbent at the beginning of iteration 1 of round $\ell$.   
Denote $\bar{B}(\ell,k)=1-B(\ell,k)$ and $\bar{D}(\ell)=1-D(\ell)$. 
Let $V(\ell,k)$ be $1$ if $D(\ell)=1$ and arm $1$ loses in any iteration $1$ through $k-1$ of round $\ell$. 

We may only incur weak regret during round $\ell$ iteration $k$ if $\bar{D}(\ell)=1$, or if $V(\ell,k')=1$ for some $k'<k$.
We will separately bound the regret incurred in these two different scenarios.
Moreover, our bound on the number of pulls, and thus the regret incurred, in this iteration will depend on whether $B(\ell,k)=1$ or $\bar{B}(\ell,k)=1$.  
This leads us to state four inequalities in the following pair of lemmas, which we will in turn use to show Theorem~\ref{main_weak}.  The first lemma applies in both the total order and Condorcet settings, while the second applies only in the total order setting.  When proving Theorem~\ref{main_condorcet} we replace Lemma~\ref{inequ2} by an alternate pair of inequalities.

\newcommand{\Arms}{A}

\begin{lemma}
\begin{align}
&\mathbb{E}[\bar{D}(\ell)B(\ell,k)T_{\ell,k}] \leq \frac{1}{2p-1}\left(\frac{1-p}{p}\right)^{\ell-1}, \nonumber \\
&\mathbb{E}[V(\ell,k)B(\ell,k)T_{\ell,k}] \leq \frac{1}{2p-1}\left(\frac{1-p}{p}\right)^{\ell}. \nonumber 
\end{align}
\label{inequ1}
\end{lemma}

\begin{lemma}
Under the total order assumption: 
\begin{itemize}
\item $\mathbb{E}\left[\sum_{k=1}^{N-1}\bar{D}(\ell)\bar{B}(\ell,k)T_{\ell,k}\right]$
is bounded above by 
$\left(\frac{1-p}{p}\right)^{\ell-1}\frac{2N\ell p^2}{(2p-1)^4}(\log(N)+1)$.
\item 
$\mathbb{E}\left[\sum_{k=1}^{N-1}V(\ell,k)\bar{B}(\ell,k)T_{\ell,k}\right]$ is bounded above by 
$\left(\frac{1-p}{p}\right)^{\ell}\frac{2N\ell p^2}{(2p-1)^4}(\log(N)+1)$. 
\end{itemize}
\label{inequ2}
\end{lemma}
We now state our main result for the total order setting, which shows that the expected cumulative weak regret is $O\left(\frac{N\log(N)}{(2p-1)^{5}}\right)$.

\begin{theorem}
The expected cumulative weak regret of \WSWEAK\ is bounded by $\left[\frac{2p^3}{(2p-1)^6}N(\log(N)+1)+\frac{N}{(2p-1)^2}\right]$ under the total order assumption.
\label{main_weak}
\end{theorem}
\begin{proof}
Iterations can be divided into two types: those in which the incumbent is better than the challenger, and those where the incumbent is worse.

We first bound expected total weak regret incurred in the first type of iteration, and then below bound that incurred in the second type.  In this first bound, observe that we incur weak regret during round $\ell$ if $D(\ell)=0$, or if $D(\ell)=1$ but arm $1$ loses to some other arm during this round. Under the second scenario, we do not incur any regret until arm $1$ loses to another arm. 

Thus, the expected weak regret incurred during iterations with a better incumbent is bounded by 
\begin{align}
&\mathbb{E}\left[\sum_{\ell=1}^{\infty}\sum_{k=1}^{N-1}B(\ell,k)T_{\ell,k}\bar{D}(\ell)+\sum_{\ell=1}^{\infty}\sum_{k=1}^{N-1}B(\ell,k)T_{\ell,k}V(\ell,k)\right]. \nonumber 
\end{align}

The first part of this summation can be bounded by the first inequality in Lemma~\ref{inequ1} to obtain
\begin{align}
\mathbb{E}&\left[\sum_{\ell=1}^{\infty}\sum_{k=1}^{N-1}B(\ell,k)T_{\ell,k}\bar{D}(\ell)\right] \nonumber \\
&\leq \sum_{\ell=1}^{\infty}\left(\frac{1-p}{p}\right)^{\ell-1}\frac{N}{2p-1}= \frac{pN}{(2p-1)^2}. \nonumber 
\end{align}

The second part of this summation can be bounded by the second inequality in Lemma~\ref{inequ1} to obtain
\begin{align}
\mathbb{E}&\left[\sum_{\ell=1}^{\infty}\sum_{k=1}^{N-1}B(\ell,k)T_{\ell,k}V(\ell,k)\right] \nonumber \\
&\leq \sum_{\ell=1}^{\infty}\frac{N}{2p-1}\left(\frac{1-p}{p}\right)^{\ell} \nonumber = \frac{N(1-p)}{(2p-1)^2}. \nonumber 
\end{align}

Thus, the cumulative expected weak regret incurred during iterations with a better incumbent is bounded by $\frac{N}{(2p-1)^2}$.

Now we bound the expected weak regret incurred during iterations where the incumbent is worse than the challenger.  This is bounded by
\begin{align}
&\mathbb{E}\left[\sum_{\ell=1}^{\infty}\sum_{k=1}^{N-1}\bar{B}(\ell,k)T_{\ell,k}\bar{D}(\ell)+\sum_{\ell=1}^{\infty}\sum_{k=1}^{N-1}\bar{B}(\ell,k)T_{\ell,k}V(\ell,k)\right]. \nonumber
\end{align}

The first term in the summation can be bounded by the first inequality of Lemma~\ref{inequ2} to obtain
\begin{align}
\mathbb{E}&\left[\sum_{\ell=1}^{\infty}\sum_{k=1}^{N-1}\bar{B}(\ell,k)T_{\ell,k}\bar{D}(\ell)\right] \nonumber \\
&\leq  \sum_{\ell=1}^{\infty}\frac{2N\ell p(1-p)}{(2p-1)^4}(\log(N)+1)\left(\frac{1-p}{p}\right)^{\ell-1} \nonumber \\
&=  \frac{2Np^{4}}{(2p-1)^6}(\log(N)+1). \nonumber
\end{align}

The second term in the summation can be bounded by the first inequality of Lemma~\ref{inequ2} to obtain
\begin{align}
\mathbb{E}&\left[\sum_{\ell=1}^{\infty}\sum_{k=1}^{N-1}\bar{B}(\ell,k)T_{\ell,k}V(\ell,k)\right] \nonumber \\
&\leq \sum_{\ell=1}^{\infty}\frac{2N\ell p^2}{(2p-1)^4}(\log(N)+1)\left(\frac{1-p}{p}\right)^{\ell} \nonumber \\
&= \frac{2p^3(1-p)}{(2p-1)^6}N(\log(N)+1). \nonumber
\end{align}

Thus, the cumulative expected weak regret incurred during iterations with a worse incumbent is bounded by $\frac{2p^3}{(2p-1)^6}N(\log(N)+1)$.

Summing these two bounds, the cumulative expected weak regret is bounded by $\left[\frac{2p^3}{(2p-1)^6}N(\log(N)+1)+\frac{N}{(2p-1)^2}\right]$. \qedhere
\end{proof}

We prove the following result 
for the Condorcet winner setting 
in a similar manner 
in the supplement.

\begin{theorem}
The expected cumulative weak regret of \WSWEAK\ is bounded by $\frac{N}{(2p-1)^2}+\frac{pN^2}{(2p-1)^3}$ under the Condorcet winner setting.
\label{main_condorcet}
\end{theorem}

\subsection{Winner Stays with Strong Regret (WS-S)}
\label{WS2}

\begin{figure*}[!t]
    \centering
    \begin{subfigure}[t]{0.5\textwidth}
        \centering
        \includegraphics[width=1\textwidth]{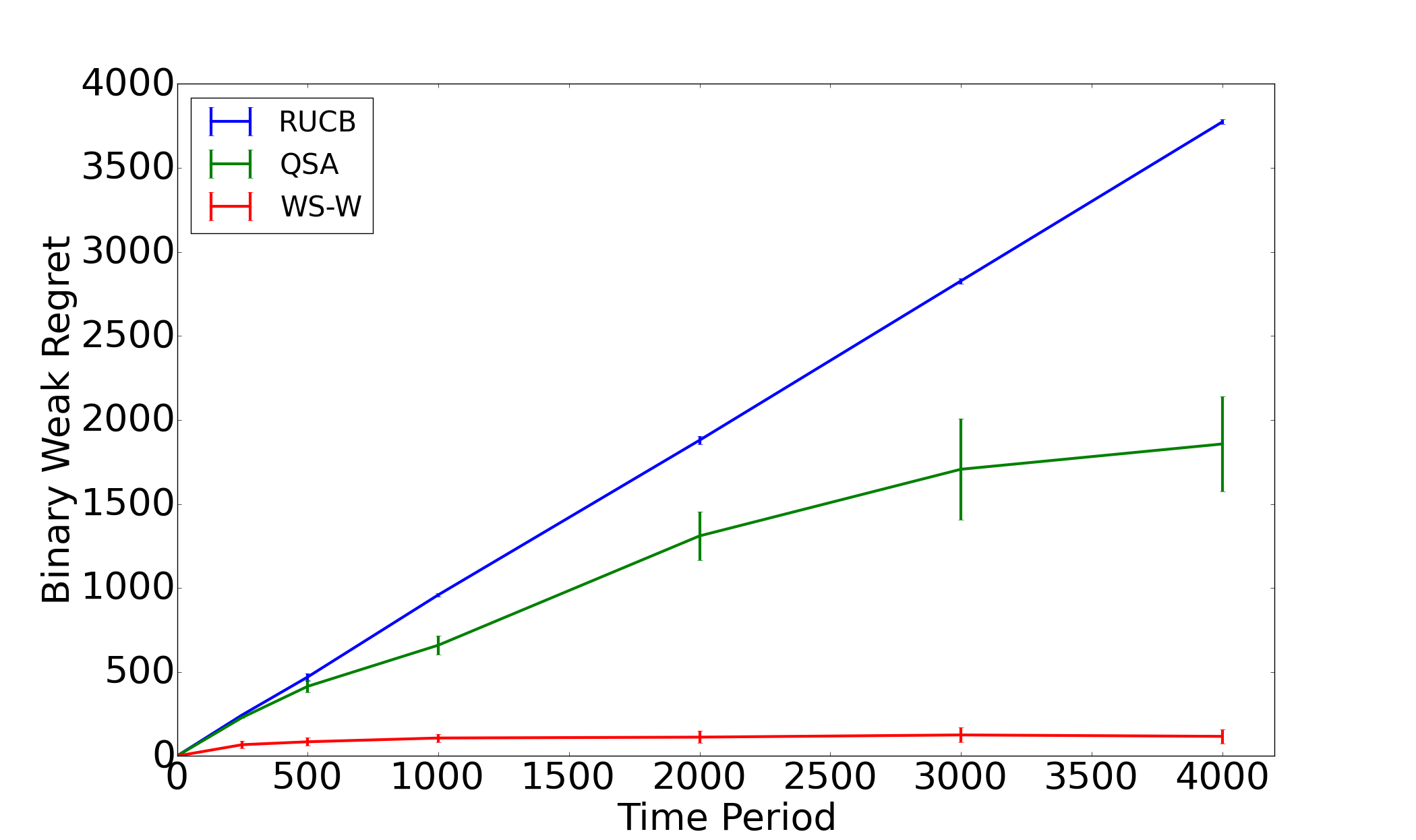}
        \caption{Simulated Dataset}   
        \label{fig:simulated_data} 
        \end{subfigure}%
    ~ 
    \begin{subfigure}[t]{0.5\textwidth}
        \centering
        \includegraphics[width=1\textwidth]{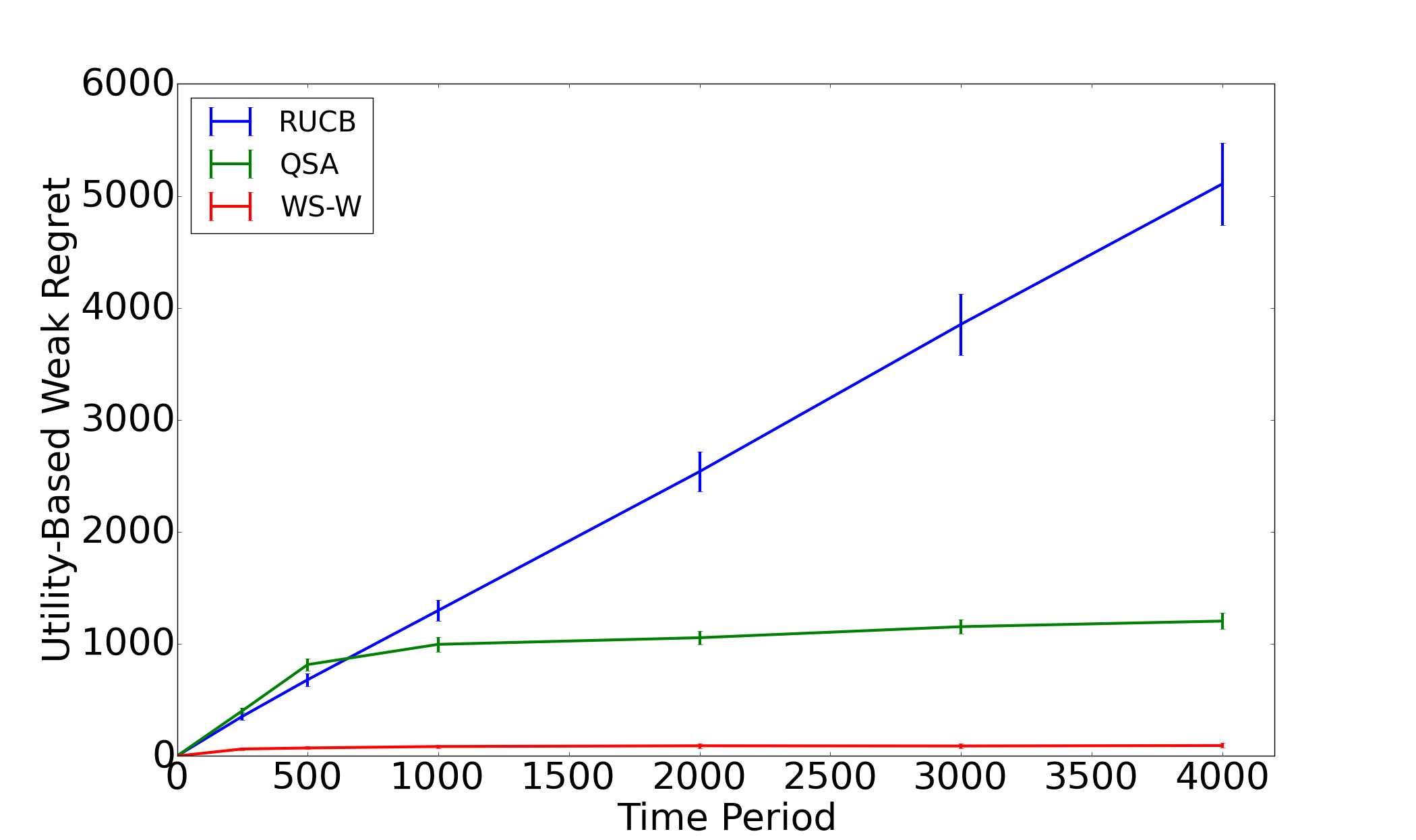}
        \caption{Yelp Academic Dataset}
        \label{fig:yelp_data}
    \end{subfigure}
    \caption{Comparison of the weak regret between \WSWEAK, RUCB and QSA using simulated data, and the Yelp academic dataset. In both experiments, \WSWEAK\ outperforms RUCB and QSA, provided constant expected cumulative weak regret.}
\label{fig:result2}
\end{figure*}

In this section, we define a version of \WS\ for strong regret, \WSSTRONG, which uses \WSWEAK\ as a subroutine.  \WSSTRONG\ is defined in Algorithm~\ref{algostrong}

\begin{algorithm}
Input: $\beta>1$, arms $1,\cdots, N$\\
 \For{$\ell=1,2,\cdots$}{
    Exploration phase: Run the $\ell^{th}$ round of \WSWEAK.\\
    Exploitation phase: Let $Z(\ell)$ be the index of the best arm at the end of the $\ell^{th}$ round. For the next $\floor{\beta^{\ell}}$ time periods, pull arms $Z(\ell)$ and $Z(\ell)$ and ignore the feedback.
 }
 \caption{\WSSTRONG}
 \label{algostrong}
\end{algorithm}

Each round of \WSSTRONG\ consists of an exploration phase and an exploitation phase. The length of the exploitation phase increases exponentially with the number of phases. Changing the parameter $\beta$ balances the lengths of these phases, and thus balances between exploration and exploitation.  Our theoretical results below guide choosing $\beta$.

We now bound the cumulative strong regret of this algorithm under both the total order and Condorcet winner settings:
\begin{theorem}
If there is a total order among arms, then for $1<\beta<\frac{p}{1-p}$, the expected cumulative strong regret for \WSSTRONG\ is bounded by $\left[\frac{2p^3}{(2p-1)^6}N(\log(N)+1)+\frac{N\log_{\beta}(T(\beta-1))}{2p-1}\right]$.
\label{t:strong_total}
\end{theorem}
\begin{proof}
Suppose at time T, we are in round $\ell$. Then 
$\beta + \cdots + \beta^{\ell} \leq T$.
Solving for $\ell$, we obtain $\ell\leq \log_{\beta}(T(\beta-1))$.

We bound the expected strong regret up to time $T$. The expected regret can be divided in two parts: the regret occuring during the exploration phase; and the regret occuring during the exploitation phase.

First we focus on regret incurred during exploration.  We never pull the same arm twice during this phase, and so regret is incurred in each time period.  To bound regret incurred during exploration, we bound the length of time spent in this phase.

The length of time spent in exploration up to the end of round $\ell$ with a better incumbent is bounded by $\frac{(N-1)\ell}{2p-1}$. 
The length of time spent with a worse incumbent, based on the proof of Theorem~\ref{main_weak}, is bounded by $\frac{2p^3}{(2p-1)^6}N(\log(N)+1)$.

Now we focus on regret incurred during exploitation.
The probability we have identified the wrong arm at the end of the $i^{th}$ round is less than $\left(\frac{1-p}{p}\right)^{i}$. Thus, the expected regret incurred during this phase up until the end of the $\ell^{th}$ round is bounded by
$\sum_{i=1}^{\ell}\left(\frac{1-p}{p}\right)^{i}\times\beta^{i} \leq \ell$.

Overall, this implies that the strong expected regret up to time $T$ (recall that $T$ is in round $\ell$) is bounded by 
\begin{align}
&\left[\frac{2p^3}{(2p-1)^6}N(\log(N)+1)+\ell+\frac{(N-1)\ell}{2p-1}\right] \nonumber \\
\leq &\left[\frac{2p^3}{(2p-1)^6}N(\log(N)+1)+\frac{N\log_{\beta}(T(\beta-1))}{2p-1}\right]. \nonumber 
\end{align}
Thus, the expected strong regret up to time $T$ is $O(N\log(T)+N\log(N))$.
\end{proof}

\begin{theorem}
Under the Condorcet winner setting and for $1<\beta<\frac{p}{1-p}$, the expected cumulative strong regret for \WSSTRONG\ is bounded by $\left[\frac{N^2 p}{(2p-1)^2}+\frac{N\log(T(\beta-1))}{(2p-1)\log(\beta)}\right]$.
\label{t:strong_condorcet}
\end{theorem}
\begin{proof}
The proof mirrors that of Theorem~\ref{t:strong_total}, with the only difference being that we bound the length of exploration with a worse incumbent using the proof of Theorem~\ref{main_condorcet} rather than Theorem~\ref{main_weak}, and the bound is $O(N^2)$.  Due to its similarity, the proof is omitted.
\end{proof}

These results provide guidance on the choice of $\beta$.  If $\beta$ is too close to 1, then we spend most of the time in the exploration phase, which is guaranteed to generate strong regret.  The last inequality in the proof of Theorem~\ref{t:strong_total} suggests that asymptotic regret will be smallest if we choose $\beta$ as large as possible without going beyond the $p/(1-p)$ threshold.  Indeed, if $\beta$ is too large, then \WSSTRONG\ may incur large regret in early exploitation stages when we have finished only a few rounds of exploration.  In our numerical experiments we set $\beta = 1.1$, which satisfies the $p/(1-p)$ constraint assumed by our theory if $p>\beta/(1+\beta) \approx .524$.  With a properly chosen $\beta$, the numerical experiments in section~\ref{strongregret} suggest \WSSTRONG\ performs better than previously devised algorithms. 
At the same time, the best choice of $\beta$ is dependent on $p$.
Modifying \WSSTRONG\ to eliminate parameters that must be chosen with knowledge of $p$ is left for future work.  

Our regret bound grows as $p$, which is the minimal gap between two arms, shrinks, and $p$ tends to decrease as the number of arms $N$ increases. Other dueling bandit algorithm for strong regret, such as RUCB and RMED, have regret bounds with better dependence on the gaps between arms. Modifying \WSSTRONG\ to provide improved dependence on these gaps is also left for future work.

\subsection{Extension to Utility-Based Regret}
\label{utility}

We now briefly discuss utility-based extensions of weak and strong regret for the total order setting, following utility-based bandits studied in \citet{ailon2014reducing}.  Our regret bounds also apply here, with a small modification.

Suppose that the user has a utility $u_i$ associated with each arm $i$.  Without loss of generality, we assume $u_1 > u_2 > \cdots > u_N$, and as in the total order setting, we require that $p_{i,j} > 0.5$ when $i<j$.  Typically the $p_{i,j}$ would come from the utilities of arms $i$ and $j$ via a generative model.  We give an example in our numerical experiments.

Then, the single-period {\it utility-based weak regret} is $r(t) =  u_{1}-\max\{u_{i_{t}}, u_{j_t}\}$, which is the difference in utility between the best arm overall and the best arm that the user can choose from those offered.  The single-period {\it utility-based strong regret} is $r(t) = u_1 - \frac{u_{i_t}+ u_{j_t}}{2}$. To get zero regret under strong regret, the best arm must be pulled twice.

Our results from Section~\ref{Methods} carry through to this more general regret setting.
Let $R=u_1 - u_N$ be the maximum single-period regret. Then, the expected cumulative utility-based weak regret for \WSWEAK\ is $O\left(R\frac{N\log(N)}{(2p-1)^5}\right)$, and the expected cumulative utility-based strong regret for \WSSTRONG\ is $O(R\left[N\log(T)+N\log(N)\right])$.

\section{Numerical Experiments}
\label{Numerical}

In this section, we evaluate \WS\ under both the weak and strong regret settings, considering both their original (binary) and utility-based versions.  In the weak regret setting, we compare \WSWEAK\ with RUCB and QSA. In the strong regret setting, we compare \WSSTRONG\ with 7 benchmarks including RUCB and Relative Minimum Empirical Divergence (RMED) by \citet{komiyama2015regret}. 
We also include an experiment violating the total order assumption in Section 11 in the supplement.  
\WS\ outperforms all benchmarks tested in these numerical experiments.



\begin{figure*}[t]
    \centering
    \begin{subfigure}[h!]{0.5\textwidth}
        \centering
        \includegraphics[width=1\textwidth]{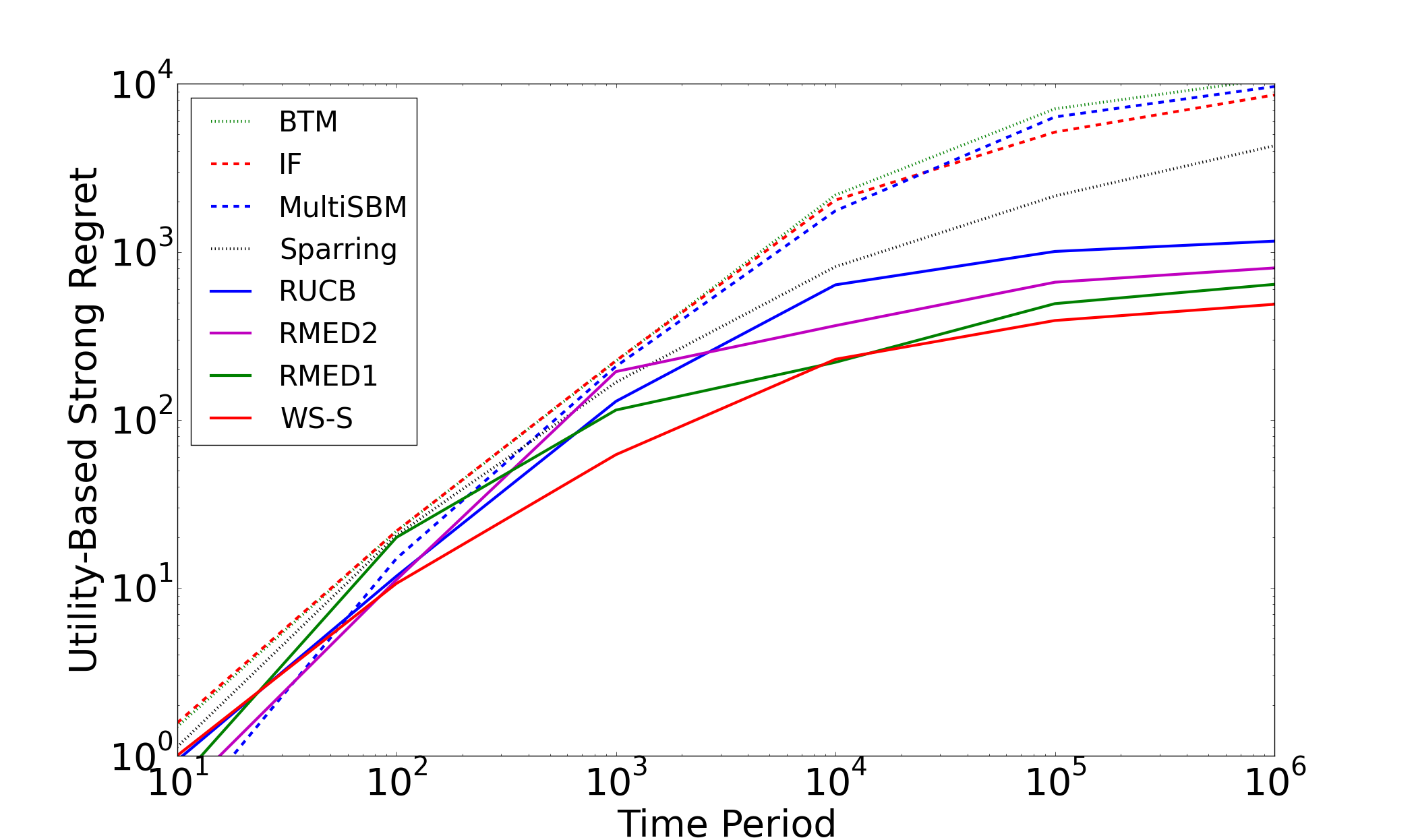}
        \caption{Sushi dataset with utility-based strong regret}
        \label{fig:sushi2}
    \end{subfigure}%
    ~
    \begin{subfigure}[h!]{0.5\textwidth}
        \centering
        \includegraphics[width=1\textwidth]{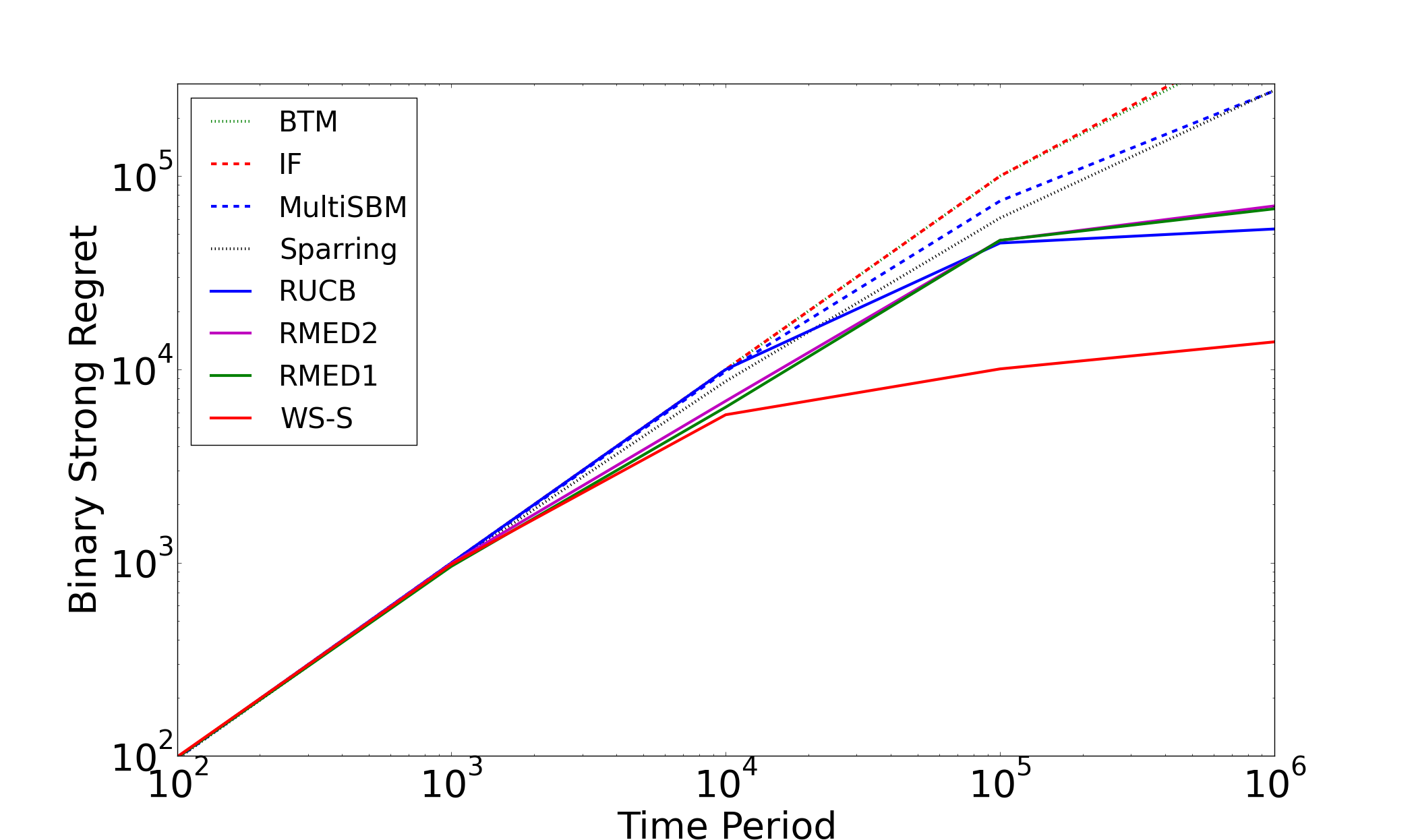}
        \caption{Sushi dataset with binary strong regret}   
        \label{fig:sushi1} 
    \end{subfigure} 
    \\
    \begin{subfigure}[h!]{0.5\textwidth}
        \centering
        \includegraphics[width=1\textwidth]{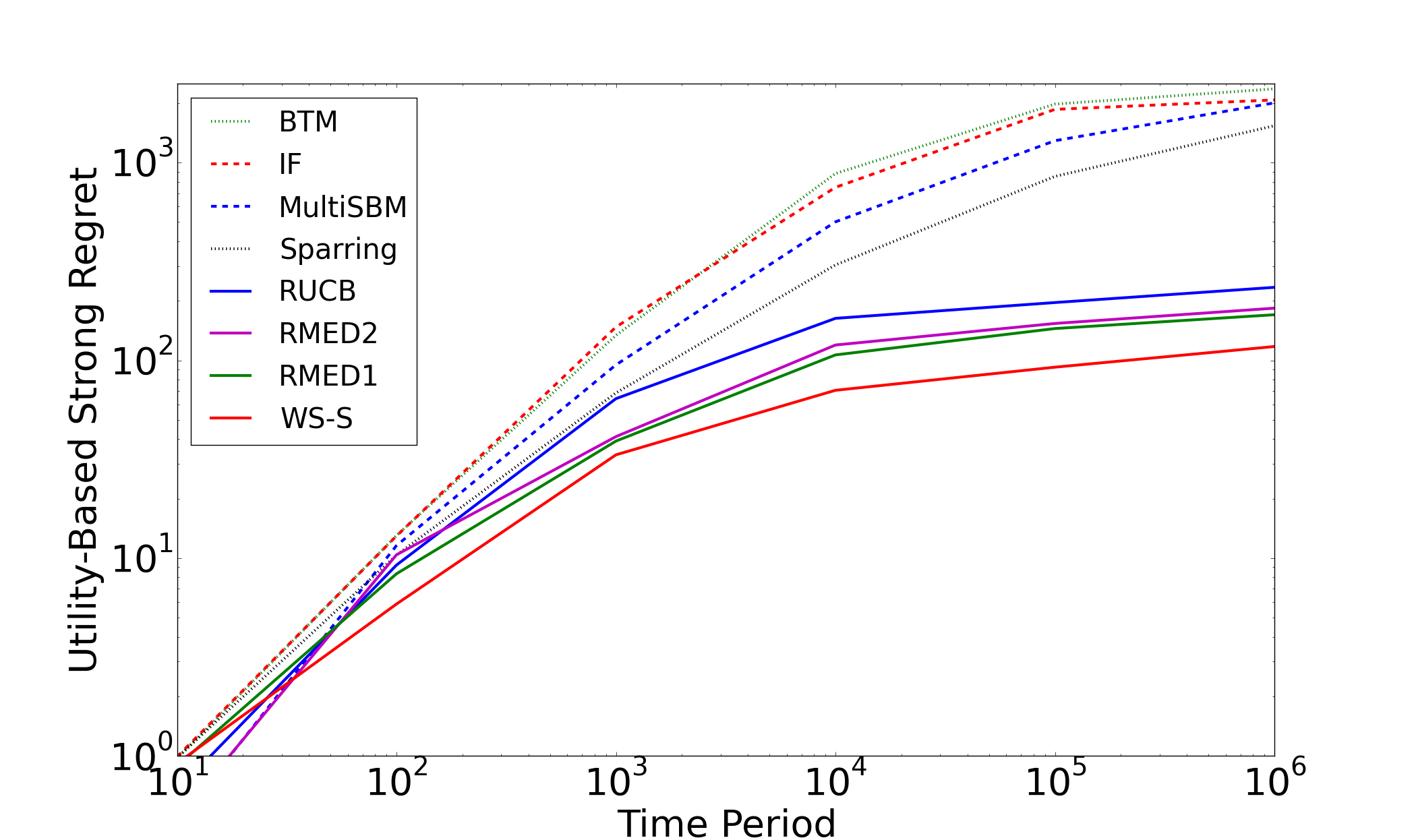}
        \caption{MSLR dataset with utility-based strong regret}
        \label{fig:MSLR2}
    \end{subfigure}%
    ~
    \begin{subfigure}[h!]{0.5\textwidth}
        \centering
        \includegraphics[width=1\textwidth]{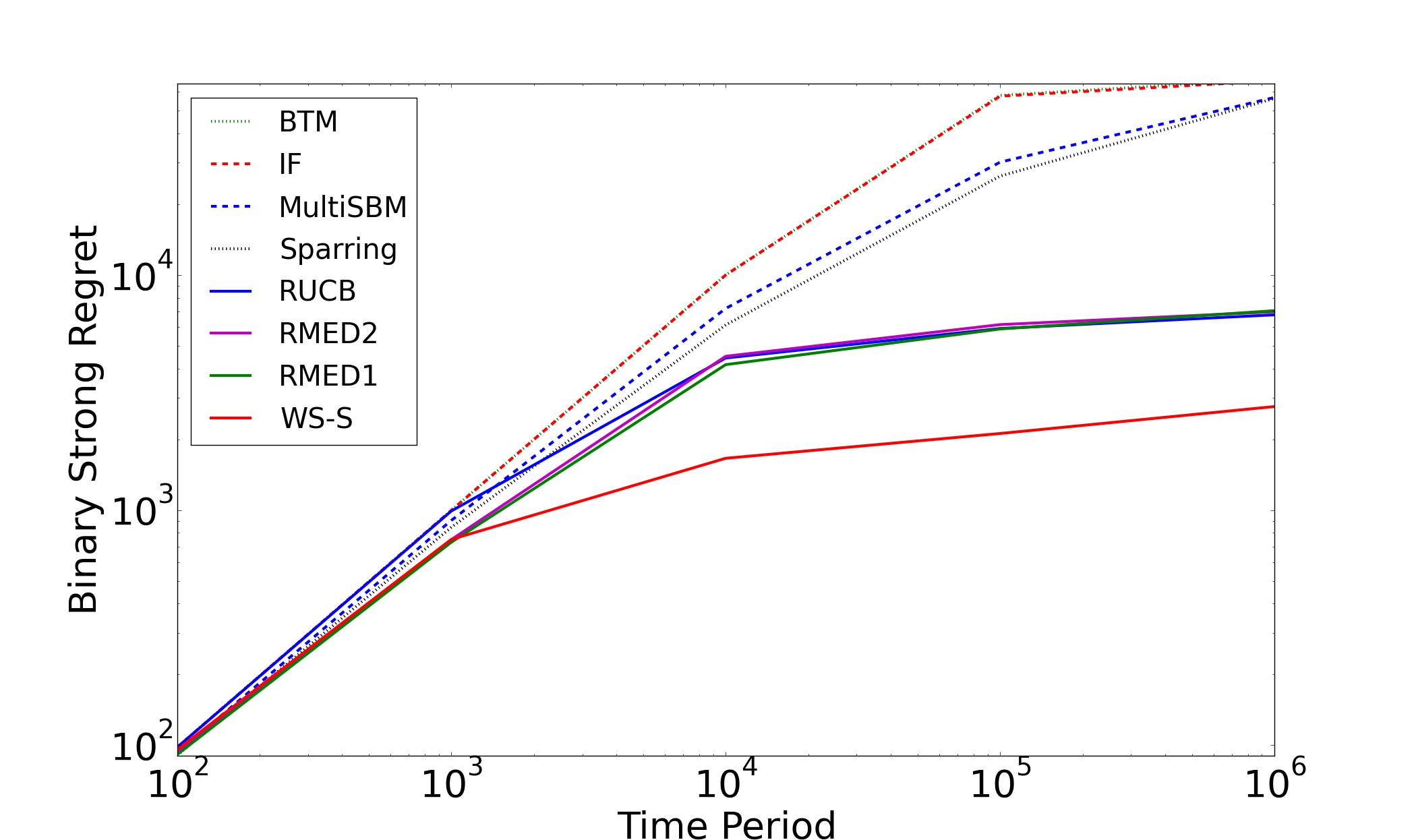}
        \caption{MSLR dataset with binary strong regret}   
        \label{fig:MSLR1} 
    \end{subfigure}
    \caption{Comparison of the strong regret between \WSSTRONG\  and 7 benchmarks on the sushi and MSLR datasets. For utility-based strong regret, we start our plot from $t=10$ since the performance of all algorithms are close to each other before $t=10$. For the same reason, we start our plot from $t=100$ for the binary strong regret. \WSSTRONG\ outperforms all benchmarks in all settings studied. 
\label{fig:result3}}
\end{figure*}

\subsection{Weak Regret}

We now compare \WSWEAK\ with QSA and RUCB using simulated data and the Yelp academic dataset \cite{yad}.

\subsubsection{Simulated Data}
\label{sec:simu_data}
In this example, we compare \WSWEAK\ with RUCB and QSA on a problem with $50$ arms and binary weak regret. Each arm is a 20-dimensional vector uniformly generated from the unit circle. We assume $p_{i,j}\!=\!0.8$ for all $i\!<\!j$.

The results are summarized in Figure~\ref{fig:simulated_data}. RUCB has approximately linear regret over the time horizon pictured. This is common in the dueling bandits literature, where many algorithms require $\sim 10^4$ comparisons before they achieve $\log(T)$ cumulative regret for $50$ arms. \WSWEAK\ finds the optimal arm after $\sim\! 500$ comparisons and has a regret that is consistent with our theoretically established constant expected cumulative weak regret.

\subsubsection{Yelp Academic Dataset}
\label{realdata}
In this example, we compare \WSWEAK\ with RUCB and QSA using the Yelp academic dataset \cite{yad} and utility-based weak regret. 

We choose $100$ restaurants from Las Vegas as our arms. Associated with each arm (restaurant) $i$ is a $20$-dimensional feature vector $A_i$, calculated using doc2vec \cite{rehurek2010software} from its reviews. We select 49 users who have reviewed at least $20$ of these $100$ restaurants.  For each user, we model their utility for restaurant $i$ as $u_i = A_i \cdot \theta$, where $\theta$ is a $20$-dimensional vector of preferences.  We infer $\theta$ for each user using linear regression.  

To model $p_{i,j}$, we then use the probit model.
We let $\hat{\sigma}^{2}$ be the estimated variance of the residuals from the linear regression above. When presented with two restaurants, we model the user as taking independent random draws from a normal distribution with means $u_i$ and $u_j$ respectively and variances $\hat{\sigma^2}$, and choosing the restaurant with the larger draw.  This gives $p_{ij} = \Phi(u_i - u_j)$, where $\Phi(\cdot)$ is the cdf for the normal distribution with mean 0 and variance $2\hat{\sigma}^{2}$.



We simulate performance for each user separately, and then average the results. These results are summarized in Figure~\ref{fig:yelp_data}. \WSWEAK\ outperforms RUCB and QSA, finding the optimal restaurant after $\sim 500$ iterations.

\subsection{Strong Regret}
\label{strongregret}

In this section, we compare \WSSTRONG\ using binary and utility-based strong regret with 7 benchmarks from the literature.  We use the sushi and MSLR datasets, which were previously used by 
\citet{komiyama2016copeland} and \citet{zoghi2015copeland} respectively to evaluate dueling bandit algorithms.  

The sushi dataset \cite{komiyama2016copeland} contains 16 arms corresponding to types of sushi, with pairwise preferences inferred from data on sushi preferences from 5000 users in \citet{kamishima2003nantonac}. The MSLR dataset has 5 arms, corresponding to ranking algorithms, with pairwise preferences provided in \citet{zoghi2015copeland}. We give preference matrices $(p_{i,j})$ for both datasets in the supplement. For utility-based regret, we define $u_{i}=2(1-p_{1,i})$.  

\WSSTRONG\ has a user-defined parameter $\beta$. In our experiments we set $\beta=1.1$.  The corresponding minimum $p$ for which our theoretical bounds hold is $\beta/(1+\beta) \approx 0.52$. We recommend $\beta\approx 1.1$ for problems of $20$ arms or fewer, and $\beta$ closer to 1 for those problems with more arms that are likely to have $p$ closer to $1/2$. We also conduct a sensitivity analysis of $\beta$ in the supplement.

Figure~\ref{fig:result3} shows the results of our comparisons. \WSSTRONG\ outperforms all 7 benchmarks considered on both datasets using both variants of strong regret.






\section{Conclusion}
In this paper, we consider dueling bandits for online content recommendation using both weak and strong regret.

We propose a new algorithm, \WS, with variants designed for the weak regret (\WSWEAK) and strong regret (\WSSTRONG) settings. We prove \WS\ has constant weak regret and optimal strong regret in $T$. In numerical experiments, \WS\ outperforms all benchmarks considered on both simulated and real datasets.

\section*{Acknowledgements} 
The authors were partially supported by NSF CAREER CMMI-1254298, NSF CMMI-1536895, NSF IIS-1247696,  NSF DMR-1120296,
AFOSR FA9550-12-1-0200, AFOSR FA9550-15-1-0038,  and AFOSR FA9550-16-1-0046.

\newpage
\nocite{langley00}

\newcommand{\newblock}{}

\bibliography{example_paper}
\bibliographystyle{icml2017}

\newpage

\title{Supplementary Materials: \\ Dueling Bandits with Weak Regret}

\maketitle

\appendix

\section{Gambler's Ruin Lemma}
In our analysis of \WSWEAK, we will use results from a special case of the Gambler's ruin problem \cite{fcsp}, stated as follows: suppose a gambler has $m$ dollars initially. In each of a sequence of rounds, he loses $1$ dollar with probability $q\neq \frac{1}{2}$ and wins $1$ dollar with probability $1-q$. He stops playing when he has either $m+1$ dollars or has no money left. We have the following result, with a proof available on Page 73 of \citet{fcsp}.

\begin{lemma}[Gambler's Ruin Lemma]
\label{ruin}
In the gambler's ruin problem: (1) the probability that the gambler reaches $m+1$ dollars before reaching $0$ dollars is $q_{m}=\frac{\left(\frac{1-q}{q}\right)^{m}-1}{\left(\frac{1-q}{q}\right)^{m+1}-1}$; (2) the expected number of steps before the gambler stops playing is $\frac{m}{1-2q}-\frac{m+1}{1-2q}\frac{\left(\frac{1-q}{q}\right)^{m}-1}{\left(\frac{1-q}{q}\right)^{m+1}-1}$.
\end{lemma}

Observe that the conditional distribution of $T_{\ell,k}$ and the winner of iteration $k$ round $\ell$, given the two arms being pulled, is given by the result above for the Gambler's ruin problem.  We leverage this in our proof.

\section{Proof of Lemma~\ref{adv}}

\begin{proof}
Suppose we are comparing arm $i$ versus arm $j$ in this iteration with $i>j$ and arm $i$ is the incumbent. Then we know $C(t_{\ell,k}-1,i)=(N-1)(\ell-1)+k-1$ and $C(t_{\ell,k}-1,j)=-\ell+1$. We will keep playing these two arms until $C(t_{\ell,k}+T_{\ell,k}-1,i)=(N-1)(\ell-1)+k$ or $C(t_{\ell,k}+T_{\ell,k}-1,j)=(N-1)(\ell-1)+k$. Further, since the winning probability of arm $i$ over arm $j$ is $p_{i,j}$ over this period, we know the dynamics of this iteration are the same as those of the Gambler's Ruin problem. Denote $E=C(t_{\ell,k}-1,i)-C(t_{\ell,k}-1,j)+1=Nl+k-N$. Then the expected length of time we spend in this iteration by Lemma~\ref{ruin} is
\begin{align}
&\frac{E}{1-2p_{i,j}}-\frac{E+1}{1-2p_{i,j}}\frac{\left(\frac{1-p_{i,j}}{p_{i,j}}\right)^{E}-1}{\left(\frac{1-p_{i,j}}{p_{i,j}}\right)^{E+1}-1} \nonumber \\
\leq & \frac{E}{1-2p_{i,j}}\leq  \frac{E}{2p-1}. \nonumber 
\end{align}

The proof of second statement is similar. Using the same notation but now supposing $p_{i,j}\geq p>\frac{1}{2}$, we have that the expected length of time we spend in this iteration is
\begin{align}
&\frac{E}{1-2p_{i,j}}-\frac{E+1}{1-2p_{i,j}}\frac{\left(\frac{1-p_{i,j}}{p_{i,j}}\right)^{E}-1}{\left(\frac{1-p_{i,j}}{p_{i,j}}\right)^{E+1}-1} \nonumber \\
=&\frac{1}{2p_{i,j}-1}-\frac{E+1}{1-2p_{i,j}}\frac{p_{i,j}(1-p_{i,j})^{E}-(1-p_{i,j})^{E+1}}{(1-p_{i,j})^{n+1}-p_{i,j}^{E+1}} \nonumber \\
\leq &\frac{1}{2p-1}. \nonumber
\end{align}
\end{proof}

\section{Proof of Lemma~\ref{count}}

In this section, we prove Lemma~\ref{count} from the main paper. This section is structured as follows: In section~\ref{preliminary}, we provide two bounds for the incumbent's losing and winning probability; In section~\ref{bound:gbm}, we consider a version of the problem in which better and  worse incumbents have constant (but different) winning probabilities and provide a upper bound for the number of worse incumbents in a round before a better incumbent loses ; In section~\ref{bound-reality}, we use the results from the previous subsection to bound the expected number of iterations with a worse incumbent in a single round before a better incumbent loses, starting from within a round; 
In section~\ref{bound:N}, 
we prove a similar bound on the expected number of iterations with a worse incumbent in this and future rounds before a better incumbent loses, starting from the beginning of a round;
In section~\ref{case2}, we complete the proof of Lemma~\ref{count}.

Throughout this section, we use a one to one correspondence between $n$ and $(\ell,k)$ defined by $n=(\ell-1)(N-1)+k$, $0\leq k\leq N-1$ and $\ell=\lceil n/(N-1) \rceil$. We also denote $p^{*}=\frac{2p-1}{p}$.

\subsection{Bounds on Win and Loss Probabilities}
\label{preliminary}
We first prove the following two lemmas, which give
\begin{itemize}
\item a lower bound for the probability that a worse incumbent loses an iteration;
\item an upper bound for the probability that a better incumbent loses an iteration.
\end{itemize}

\begin{lemma}
In iteration $k$ of round $\ell$ conditioned on the identities of the incumbent and the challenger, if the incumbent is worse than the challenger, then the incumbent loses the iteration with 
conditional probability at least $p^{*}=\frac{2p-1}{p}$.
\label{lower}
\end{lemma}
\begin{proof}
Let $i$ be the incumbent and $j$ be the challenger, with $i>j$.
$C(i,t_{\ell,k}) \ge 0$ and $C(j,t_{\ell,k}) \le 0$.  Let $E=C(i,t_{\ell,k})+|C(j,t_{\ell,k})|+1$.
The probability that arm $i$ loses this iterations is the same as $1-q_{E}$ in the Gambler's Ruin Lemma, Lemma~\ref{ruin}, with $q=p_{i,j}<0.5$. 
This probability is:
\begin{align}
1-q_{E}&=1-\frac{\left(\frac{1-p_{j,i}}{p_{i,j}}\right)^{E}-1}{\left(\frac{1-p_{i,j}}{p_{i,j}}\right)^{E+1}-1} \nonumber \\
&\geq\frac{\left(\frac{1-p_{i,j}}{p_{i,j}}\right)^{E+1}-\left(\frac{1-p_{i,j}}{p_{i,j}}\right)^{E}}{\left(\frac{1-p_{i,j}}{p_{i,j}}\right)^{E+1}} 
=\frac{1-2p_{i,j}}{1-p_{i,j}} \nonumber \\
&\geq \frac{2p-1}{p}. \nonumber
\end{align}
\end{proof}

\begin{lemma}
In iteration $k$ of round $\ell$ conditioned on the identities of the incumbent and the challenger, if the incumbent is better than the challenger, then the incumbent loses the iteration with 
conditional probability at most $\left(\frac{1-p}{p}\right)^E$, where $E=N(\ell-1)+k$.
\label{upper}
\end{lemma}
\begin{proof}
This proof is similar to the previous one. Suppose we are pulling arm $i$ and $j$ with $i<j$ and $i$ is the incumbent. Then we know $C(t_{\ell,k}-1,i)=(N-1)(\ell-1)+k-1$ and $C(t_{\ell,k}-1,j)=-\ell+1$. The probability that arm $i$ loses is equal to $1-q_{E}$ from the gambler's ruin problem, where $E=(N-1)(\ell-1)+k-1+\ell-1=N(\ell-1)+k$. We have
\begin{align}
1-q_{E}&=1-\frac{\left(\frac{1-p_{i,j}}{p_{i,j}}\right)^{E}-1}{\left(\frac{1-p_{i,j}}{p_{i,j}}\right)^{E+1}-1} \nonumber \\
&=\frac{\left(\frac{1-p_{i,j}}{p_{i,j}}\right)^{E}[1-\frac{1-p}{p}]}{1-\left(\frac{1-p_{i,j}}{p_{i,j}}\right)^{E+1}} \nonumber \\
&\leq \left(\frac{1-p_{i,j}}{p_{i,j}}\right)^{E}\leq \left(\frac{1-p}{p}\right)^{E}. \nonumber 
\end{align}
\end{proof}

\subsection{Definition and Upper Bound for $g(b,m)$}
\label{bound:gbm}

In this section, we define a function $g(b,m)$ as follows.
First, we define $g(0,m)=0$ for any $m$. We define $g(b,m)$ for other integers $b$, $m$ satisfying $m>0$ and $0\le b \le m$ recursively, as follows:

\begin{align}
&g(b,m) \nonumber \\
=&\frac{b}{m}+\sum_{b^{'}=0}^{b-1}\frac{1}{m}p^{*}g(b^{'},m-1)+\sum_{b^{'}=b}^{m-1}\frac{1}{m}g(b,m-1) \nonumber \\
&+\sum_{b^{'}=0}^{b-1}\frac{1}{m}(1-p^{*})g(b-1,m-1) \nonumber \\
=&\frac{b}{m}+\sum_{b^{'}=0}^{b-1}\frac{1}{m}p^{*}g(b^{'},m-1)+\frac{m-b}{m}g(b,m-1) \nonumber \\
&+\frac{b}{m}(1-p^{*})g(b-1,m-1) \label{equ:recur} 
\end{align}

Intuitively, $g(b,m)$ is the expected number of future iterations in which the incumbent is worse than the challenger, starting with $m$ arms that have not dueled yet $b$ of which are better than the incumbent, when we stop counting when we reach the end of the round or when an incumbent loses to a worse challenger, in a simplified problem in which worse incumbents beat better challengers with probability $p^*$.  In our problem, this probability is not $p^*$, but is bounded below by this quantity, and in the next section we will show that $g(b,m)$ is an upper bound on an analogous quantity in our problem.

We prove the following result about $g$.

\begin{lemma}
For $0\leq b\leq m\leq N-1$, we have
\begin{align}
g(b,m)=g(b,b)\leq \frac{\log(b)+1}{p^{*}}. \nonumber
\end{align}
\label{lem:g}
\end{lemma}
\begin{proof}
Given the boundary conditions $g(0,m)=0$ for all $m$, we know Equation~\eqref{equ:recur} has a unique solution. In this proof, 
\begin{itemize}
\item We first assume $g(b,m)=g(b,b)$ for all $b\leq m$ and solve for $g(b,m)$;
\item Then we show that this $g(b,m)$ is indeed the solution for Equation~\eqref{equ:recur}, verifying that $g(b,m)$ is as claimed;
\item Finally, we show $g(b,m)\leq \frac{\log(b)+1}{p^{*}}$.
\end{itemize}

First, we solve for $g(b,m)$ with the assumption that $g(b,m)=g(b,b)$ for $b\leq m$. Setting $m=b$ in Equation~\eqref{equ:recur} provides
\begin{align}
g(b,b)=1+\sum_{b^{'}=0}^{b-1}\frac{p^{*}g(b^{'},b)}{b}+(1-p^{*})g(b-1,b-1) \label{gb-rec}.
\end{align}
Thus, we know
\begin{align}
\sum_{b^{'}=1}^{b-1}&p^{*}g(b^{'},b+1) \nonumber \\
=&\sum_{b^{'}=1}^{b-1}p^{*}g(b^{'},b) \nonumber \\
=&b\left[g(b,b)-1-(1-p^{*})g(b-1,b-1)\right].\nonumber 
\end{align}
Therefore, Equation~\eqref{gb-rec} becomes
\begin{align}
g&(b+1,b+1) \nonumber \\
=&1+\frac{b}{b+1}[g(b,b)-1-(1-p^{*})g(b-1,b-1)] \nonumber \\
&+\frac{p^{*}g(b,b)}{b+1}+(1-p^{*}g(b,b). \nonumber 
\end{align}
Re-organizing the terms, we have
\begin{align}
&g(b+1,b+1)-g(b,b) \nonumber \\
=&\frac{1}{b+1}+\frac{b}{b+1}(1-p^{*})[g(b,b)-g(b-1,b-1)]. \nonumber
\end{align}
Denote $F(b) = g(b,b)-g(b-1,b-1)$. We know $F(1)=1$. Thus, we have
\begin{align}
F(b)
=&\frac{1}{b}+\frac{b-1}{b}(1-p^{*})F(b-1) \nonumber \\
=&\frac{1}{b}+\frac{1-p^{*}}{b}+\frac{b-2}{b}(1-p^{*})^{2}F(b-2) \nonumber \\
=&\frac{1}{b}+\frac{1-p^{*}}{b} + \cdots \frac{(1-p^{*})^{b-1}}{b}. \nonumber 
\end{align}
Therefore, 
\begin{align}
g(b,b) 
=&\sum_{k=1}^{b}F(k) \nonumber \\
=&\sum_{k=1}^{b}\left[\frac{1}{k}+\frac{1-p^{*}}{k} + \cdots \frac{(1-p^{*})^{k-1}}{k}\right]. \nonumber
\end{align}
Thus, if $g(b,m)=g(b,b)$ for all $b\leq m$, we know
\begin{align}
g(b,m)=\sum_{k=1}^{b}\left[\frac{1}{k}+\frac{1-p^{*}}{k} + \cdots \frac{(1-p^{*})^{k-1}}{k}\right]. \nonumber 
\end{align}

Now we verify that this is the correct solution. We prove this by induction on $b$. For $b=1$, Equation~\eqref{equ:recur} becomes
\begin{align}
g(1,m)=\frac{1}{m}+\frac{m-1}{m}g(1,m-1). \nonumber
\end{align}
Since $g(1,1)=1$, it is easy to check $g(1,2)=g(1,3)=\cdots =g(1,N-1)=1$.

Suppose this $g(b,m)=g(b,b)$ are true for all $b\leq m$, $b\leq k$. For $b=k+1$, Equation~\eqref{equ:recur} becomes
\begin{align}
&g(k+1,m) \nonumber \\
=&\frac{k+1}{m}+\sum_{b^{'}=0}^{k}\frac{p^{*}}{m}g(b^{'},m-1)+\frac{m-k-1}{m}g(k+1,m-1) \nonumber \\
&+\frac{k+1}{m}(1-p^{*})g(k,m-1)\nonumber \\
=&\frac{k+1}{m}+\sum_{b^{'}=0}^{k}\frac{p^{*}}{m}g(b^{'},b^{'})+\frac{m-k-1}{m}g(k+1,m-1) \nonumber \\
&+\frac{k+1}{m}(1-p^{*})g(k,k).\nonumber 
\end{align}

To show $g(k+1,m)$ does not depend on $m$, we need to prove the following equation is true for $m=k+2,k+3,\cdots, N-1$.
\begin{align}
&\frac{k+1}{m}+\sum_{b^{'}=0}^{k}\frac{p^{*}}{m}g(b^{'},b^{'})+\frac{k+1}{m}(1-p^{*})g(k,k) \nonumber \\
=&\frac{k+1}{m}g(k+1,m-1) \nonumber \\
\iff &k+1+\sum_{b^{'}=0}^{k}p^{*}g(b^{'},b^{'})+(k+1)(1-p^{*})g(k,k) \nonumber \\
=&(k+1)g(k+1,m-1) \label{equ:constant}
\end{align}

We first check Equation~\eqref{equ:constant} when $m=k+2$. Starting from the left hand side, we have
\begin{align}
&k+1+\sum_{b^{'}=0}^{k}g(b^{'},b^{'})+(k+1)(1-p^{*})g(k,k) \nonumber \\
=&k+1+(k+1)[g(k+1,k+1)-1-(1-p^{*})g(k,k)] \label{plugin} \\
&+(k+1)(1-p^{*})g(k,k) \nonumber \\
=&(k+1)g(k+1,k+1), \nonumber
\end{align}
which equals to the right hand side. Equation \eqref{plugin} follows from Equation \eqref{gb-rec} (Equation \eqref{gb-rec} holds because $g(b,m)=g(b,b)$ for all $b\leq k$).

Again, by induction, we know \eqref{equ:constant} is true for all $m=k+2,\cdots, N-1$ and thus we concludes our induction.

We have shown that $g(b,m)=g(b,b)$ for all $b\leq m$.

Finally, we prove $g(b,b)=g(b,m)\leq \frac{\log(b)+1}{p^{*}}$. This is because
\begin{align}
g(b,m)
=&g(b,b) \nonumber \\
=&\sum_{k=1}^{b}\left[\frac{1}{k}+\frac{1-p^{*}}{k} + \cdots \frac{(1-p^{*})^{k-1}}{k}\right]\nonumber \\
\leq & \sum_{k=1}^{b}\left[\frac{1}{k}+\frac{1-p^{*}}{k} + \cdots \frac{(1-p^{*})^{b-1}}{k}\right] \nonumber \\
=& \sum_{k=1}^{b}\frac{1}{k}\left[1+(1-p^{*})+\cdots +(1-p^{*})^{b-1}\right] \nonumber \\
\leq & \frac{\log(b)+1}{p^{*}}, \nonumber 
\end{align}
which concludes our proof.
\end{proof}

\subsection{Bound on the Number of Iterations in One Round with a Worse Incumbent, Starting from Within the Round}
\label{bound-reality}

Let $B(n)$ denote an indicator function that equals 1 if we have a better incumbent at the $n^{th}$ iteration. The definition of $B(n)$ is very similar to $B(\ell,k)$ except $B(\ell,k)$ tracks both round and iteration number. Similarly, we use $\bar{B}(n)=1-B(n)$ to denote an indicator function that equals 1 if we have a worse incumbent at the $n^{th}$ iteration. 

Let $h(i,n,\mathcal{A})$ be the expected number of iterations with an incumbent that is worse than the challenger, between iteration $n$ and the first time that a better incumbent loses to a challenger or the round ends, given that the incumbent arm at iteration $n$ is $i$ and 
$\mathcal{A}$ is the set of arms that have not yet previously dueled in the round.
Formally, we define this quantity as:
\begin{equation*}
h(i,n,\mathcal{A})
= \mathbb{E}\left[ \sum_{n'=n}^{\sigma-1} B(n') | \mathcal{A}, i_n = i \right],
\end{equation*}
where
\begin{itemize}
\item Conditioning on $\mathcal{A}$ is understood to mean that we are conditoning on $C(n-1,j) = -\ell+1\ \forall\ j \notin \mathcal{A} \cup \{i_n\}$, and $C(n-1,j) = -\ell\   \forall\ j \in \mathcal{A}$, where $\ell = \lceil n/(N-1) \rceil$ is the round in which iteration $n$ resides.  In other words, it is understood to mean that $\mathcal{A}$ contains the set of arms that have not yet dueled in this round.
\item $\sigma = \min\left\{n' > n : J(n') = 1, n' = N \lceil n / (N-1) \rceil \right\}$ where $J(n)$ is an indicator that equals 1 when a better incumbent loses at iteration $n$, i.e., $\sigma$ is the first time that either a better incumbent loses or the round ends.
\end{itemize}

\begin{lemma}
\label{lem:h}
For any $i$, $\ell$, $k$ and $\mathcal{A}$, we have
\begin{align}
h(i,n,\mathcal{A})\leq g(b,m) \leq \frac{\log(N)+1}{p^{*}}, \nonumber
\end{align}
where $m=N-k$ and $b=|\{ u \in \mathcal{A} : u < i \}|$.
\label{lemma:qij}
\end{lemma}
\begin{proof}
Denote $q_{i,j}(n)$ as the probability that incumbent arm $i$ will beat challenger $j$ at time n. We first write a recursive expression for 
$h(i,n,\mathcal{A})$ that applies when $n$ is not divisible by $N$:   
\begin{align}
h(i,n,\mathcal{A}) 
=&\sum_{\{j\in \mathcal{A}:i>j\}}\bigg[1 + \frac{q_{i,j}(n)}{N-k}h(i,n+1,\mathcal{A}\cup\{j\}) \nonumber \\
&+\frac{1-q_{i,j}(n)}{N-k}h(j,n+1,\mathcal{A}\cup\{i\})\bigg] \nonumber \\
&+\sum_{\{j\in \mathcal{A}:i<j\}}\frac{q_{i,j}(n)}{N-k}h(i,n+1,\mathcal{A}\cup{j}) \label{general-rec}.
\end{align}
When $n$ is divisible by $N-1$, the only allowed value of $\mathcal{A}$ is $\emptyset$ and 
$h(i,n,\emptyset)=0$.

We then prove the desired result via induction on the number of iterations in the round, i.e., on $n\pmod{N-1}$.  When $n\pmod{N-1} = 0$, we have $h(i,n,\emptyset)=0$, $b=0$, and $g(b,m)=0$.  Thus the result holds in this case.

Then suppose the result holds for all $n$ with a particular value of $n\pmod{N-1}$ and we show it holds for $n-1$.

Applying the induction hypothesis to the right-hand side of \eqref{general-rec}, we have
\begin{align}
h(i,n,\mathcal{A}) 
\leq & \sum_{\{j\in \mathcal{A}:i>j\}}\bigg[1 + \frac{q_{i,j}(n)}{m} g(b_{i,j},m-1) \nonumber \\
     &+ \frac{1-q_{i,j}(n)}{m} g(b_{j,j},m-1)\bigg] \nonumber \\
     &+ \sum_{\{j\in \mathcal{A}:i<j\}} \frac{q_{i,j}(n)}{m} g(b_{i,j},m-1), \label{qij-ineq0}
\end{align}
where $b_{u,j}=\#\{u^{'}\in\mathcal{A}\setminus\{j\}:u^{'}<u\}$.

Consider the summand in the first sum in \eqref{qij-ineq0}, dropping the constants $1$ and $\frac1m$,
\begin{align}
q_{i,j}(n)g(b_{i,j},m-1)+(1-q_{i,j}(n))g(b_{j,j},m-1).\label{qij-ineq1}
\end{align}
This is increasing in $q_{i,j}(n)$ when $i>j$ since $b_{i,j}>b_{j,j}$, and since $g(b,m)$ is increasing in $b$.
Since $i$ is an incumbent that is worse than the challenger when $i>j$, Lemma~\ref{lower} shows that $q_{i,j}(n) \le 1-p^* = 1-\frac{2p-1}{p}$ in this situation.
Thus, this summand is bounded above by 
$(1-p^*) g(b_{i,j},m-1)+ p^* g(b_{j,j},m-1)$.

Substituting this into \eqref{qij-ineq0}, along with the inequality $q_{i,j}(n)\leq 1$ in the last term, we have
\begin{align}
&h(i,n,\mathcal{A}) \nonumber \\
\leq & \sum_{\{j\in \mathcal{A}:i>j\}}\bigg[1+\frac{1-p^{*}}{m}g(b_{i,j},m\!-\!1)+\frac{p^{*}}{m}g(b_{j,j},m\!-\!1)\bigg] \nonumber \\
&+ \sum_{\{j\in\mathcal{A}:i<j\}}\frac{1}{m}g(b_{i,j},m-1) \nonumber \\
=&\frac{b}{m} + 
  \frac{b}{m}(1-p^{*}) g(b-1,m-1)+
  \sum_{b^{'}=0}^{b-1}\frac{p^{*}}{m}g(b^{'},m-1) \nonumber \\
&+\frac{m-b}{m}g(b,m-1) \nonumber \\
= & g(b,m) \nonumber 
\end{align}
In the second to last line we have used that $\{b_{i,j} : j \in \mathcal{A}, i > j\} = \{0,\ldots,b-1\}$ and
$b_{i,j} = b-1$ when $i>j$; 
$b_{i,j} = b$ when $i<j$; and that the cardinality of  
$\{j \in \mathcal{A} : i > j\}$ and $\{j \in \mathcal{A} : i < j\}$
are $b$ and $m-b$ respectively.
In the last line we have used the recursive definition of $g(b,m)$ in terms of $g(\cdot,m-1)$.

This shows the first inequality in the statement of the lemma.  The second inequality follows directly from Lemma~\ref{lem:g}.
\end{proof}

\subsection{Bound on the Number of Iterations with a Worse Incumbent, Starting from a Round Beginning}
\label{bound:N}
Denote  $f(i,\ell)$ to be the expected number of iterations with a worse incumbent in this and future rounds, stopping as soon as a better incumbent loses, giving that we have arm i as the incumbent at the start of round $\ell$.
\begin{lemma}
For any $i$ and $\ell$, we have
\begin{align}
f(i,\ell)\leq \frac{\log(N)+1}{(p^{*})^2}. \nonumber
\end{align}
\label{WIbound}
\end{lemma}

\begin{proof}
Let $U(i,\ell)$ denote the expected number of iterations in round $\ell$ with a worse incumbent before a better incumbent loses. We use $V(\ell)$ to denote an indicator which equals to $1$ if a better incumbent does not lose in the round $\ell$. Then for $i> 1$,
\begin{align}
f(i,\ell) = U(i,\ell)+\mathbb{E}[f(Z(\ell), \ell+1) V(\ell)|Z(\ell-1) = i]. \nonumber
\end{align}

The first term is bounded by Lemma~\ref{lem:h} by
\begin{align}
U(i,\ell)\leq \frac{\log(N)+1}{p^{*}}, \nonumber
\end{align}
for all $i$ and $\ell$. 

For the second term, since $f(Z(\ell), \ell+1) = 0$ when $Z(\ell) = 1$, we know the second term is bounded by
\begin{align} 
&\mathbb{E}[f(Z(\ell), \ell+1) V(\ell)|Z(\ell-1) = i]\nonumber \\
\leq & \mathbb{E}[ f(Z(\ell), \ell+1) | Z(\ell)\neq 1, V(\ell)=1, Z(\ell-1) = i] \nonumber \\
&\times P(Z(\ell) \neq 1, V(\ell)| Z(\ell-1) = i). \nonumber 
\end{align}
 
Let $s_j = P(Z(\ell)=j | Z(\ell)\neq 1, V(\ell), Z(\ell-1) = i)$ to be the probability distribution over the integers from $2$ through $N$. Then we know
\begin{align}
&\mathbb{E}[ f(Z(\ell), \ell+1) | Z(\ell) \neq 1,V(\ell)=1, Z(\ell-1) = i]\nonumber \\
=& \sum_{j=2}^N s_j f(j,\ell+1) \nonumber \\
\leq & \max_{j=2,…,N} f(j,\ell+1). \nonumber 
\end{align}

Further, since if arm $1$ wins its first duel as a challenger (which happens with probability at least $p^*$), then either $Z(\ell)=1$ (it wins all subsequent duel in the round) or $V(\ell)=0$ (it loses a subsequent duel), we have 
$P(Z(\ell) \neq 1, V(\ell) | Z(\ell-1) = i) \leq 1-p*$.

Thus, we know
\begin{align}
f(i,\ell)\leq \frac{\log(N)+1}{p^{*}}+(1-p^{*})\max_{j=2,\cdots,N}f(j,\ell+1). \nonumber
\end{align}

Let $f(\ell)=\max_{j=2,\cdots,N}f(j,\ell)$. Then,
\begin{align}
f(\ell)\leq \frac{\log(N)+1}{p^{*}}+(1-p^{*})f(\ell+1). \nonumber
\end{align}

Thus,
\begin{align}
f(1) 
\leq &\frac{\log(N)+1}{p^{*}}+(1-p^{*})f(2) \nonumber \\
\leq &\frac{\log(N)+1}{p^{*}}(1+(1-p^{*})+(1-p^{*})^2+\cdots) \nonumber \\
=&\frac{\log(N)+1}{(p^{*})^2}. \nonumber
\end{align}
\end{proof}

\subsection{Completing the Proof of Lemma~\ref{count}}
\label{case2}
With the lemmas in the preceding subsections established, we now complete the proof of Lemma~\ref{count}.
\begin{proof}
Let $\tau_{0}=0$ and $\tau_{k}=\{n>\tau_{k-1}:J(n)=1\}$. The expected number of iterations with a worse incumbent is 

\begin{align}
&\mathbb{E}\left[\sum_{n=0}^{\infty}\bar{B}(n)\right] \nonumber \\
=&\mathbb{E}\sum_{k=0}^{\infty}1\{\tau_{k}<\infty\}\sum_{n=\tau_{k}}^{\infty}1\{n<\tau_{k+1}\}\bar{B}(n) \nonumber \\
=&\sum_{k=0}^{\infty}P(\tau_{k}<\infty)\mathbb{E}\left[\sum_{n=\tau_{k}}^{\infty}1\{n<\tau_{k+1}\}\bar{B}(n)|\tau_{k}<\infty\right] \nonumber
\end{align}
where we have used Tonelli's theorem to exchange the expectation of an infinite sum of non-negative terms with an infinite sum of expectations of the same terms.

Conditioning on the history available at time $\tau_k$, we have that the inner expectation can be written as,
\begin{align*}
&\mathbb{E}\left[\sum_{n=\tau_{k}}^{\infty}1\{n<\tau_{k+1}\}\bar{B}(n)|\tau_{k}<\infty\right]\\
=& 
\mathbb{E}\left[
\mathbb{E}\left[
\sum_{n=\tau_{k}}^{\infty}1\{n<\tau_{k+1}\}\bar{B}(n)
|H_{\tau_k}, \tau_{k}<\infty\right]
|\tau_{k}<\infty\right],
\end{align*}
where $H_n$ is the sigma algebra 
generated by $(C(i,s) : s< t_{\ell,k^{'}}, i=1,\ldots,N)$, where $\ell=n \pmod{N-1}$, $k^{'}=\lceil n/(N-1) \rceil$,
and $H_{\tau_k}$ is the filtration $(H_n : n)$ stopped at $\tau_k$.

We further break this inner term  $\mathbb{E}\left[\sum_{n=\tau_{k}}^{\infty}1\{n<\tau_{k+1}\}\bar{B}(n)|H_{\tau_k}, \tau_{k}<\infty\right]$ into two parts: the part that occurs during the round in which $\tau_k$ resides, and the part that occurs in future rounds.
Let $\ell_k = \lceil \tau_k / (N-1) \rceil$.  Then,
\begin{align*}
&\mathbb{E}\left[
\sum_{n=\tau_{k}}^{\infty}1\{n<\tau_{k+1}\}\bar{B}(n)
|H_{\tau_k}, \tau_{k}<\infty\right]\\
= 
&\mathbb{E}\left[
\sum_{n=\tau_{k}}^{\ell_k N}
1\{n<\tau_{k+1}\}\bar{B}(n)
|H_{\tau_k}, \tau_{k}<\infty\right] \\
+ 
&\mathbb{E}\left[
\sum_{n=\ell_{k}N+1}^{\infty}
1\{n<\tau_{k+1}\}\bar{B}(n)
|H_{\tau_k}, \tau_{k}<\infty\right] \\
\le&
\frac{\log(N) + 1}{p^*}
+  \frac{\log(N) + 1}{(p^*)^2} \\
\le&  \frac{2(\log(N) + 1)}{(p^*)^2}
\end{align*}
where the second to last inequality relies on Lemma~\ref{lemma:qij} to show 
$\mathbb{E}\left[
\sum_{n=\tau_{k}}^{\ell_k N}
1\{n<\tau_{k+1}\}\bar{B}(n) 
|H_{\tau_k}, \tau_{k}<\infty\right]$ is bounded above by
$\frac{\log(N) + 1}{p^*}$
and Lemma~\ref{WIbound} to show
$\mathbb{E}\left[
\sum_{n=\ell_{k}N+1}^{\infty}
1\{n<\tau_{k+1}\}\bar{B}(n)
|H_{\tau_k}, \tau_{k}<\infty\right]$ is bounded above by
$\frac{\log(N) + 1}{(p^*)^2}$.

Thus, 
\begin{align*}
&\mathbb{E}\left[\sum_{n=0}^{\infty}\bar{B}(n)\right] 
&\le 
\frac{2(\log(N) + 1)}{(p^*)^2}
\sum_{k=0}^\infty P(\tau_k < \infty).
\end{align*}

Now we bound $P(\tau_{k}<\infty)$ for a fixed k. Based on Lemma~\ref{upper}, we know $J(n)$ is a Bernoulli random variable with success rate less than $\left(\frac{1-p}{p}\right)^{n}$ (this is because of Lemma~\ref{upper} and $n=(N-1)(\ell-1)+k<E$), independent across n. Let $Q_{n}$ denote a Bernoulli random variable with success rate $\left(\frac{1-p}{p}\right)^{n}$. Then we know:
\begin{align*}
P(\tau_{k}<\infty) 
&\leq P\left(\sum_{i=1}^{\infty}J(i)\geq k\right) \nonumber \\
& \leq P\left(\sum_{i=1}^{\infty}Q_{i}\geq k\right).
\end{align*}
Let $W_{m}=\sum_{i=1}^{m}Q_{i}$, which follows a Poisson Bernoulli distribution, and let $W=\lim_{m\rightarrow \infty} W_{m}$.   $W$ follows a Poisson distribution with parameter $\sum_{i=1}^{\infty}\left(\frac{1-p}{p}\right)^{i}=\frac{1-p}{2p-1}$ (Theorem 4, \citet{nsit}). 
Thus, 
\begin{align*}
\mathbb{E}\left[\sum_{n=0}^{\infty}\bar{B}(n)\right] 
&\le 
\frac{2(\log(N) + 1)}{(p^*)^2}
\sum_{k=0}^\infty P(W \ge k)\\
&= 
\frac{2 p^2(1-p)}{(2p-1)^{3}}(\log(N)+1)\nonumber \\
&\leq \frac{2 p^2}{(2p-1)^{3}}(\log(N)+1)
\end{align*}
\end{proof}

\section{Proof of Lemma~\ref{tail}}
\begin{proof}
It is easy to see that at the last iteration which has a worse incumbent, the better arm is always arm $1$. Thus, we only consider $C(t,1)$ in this proof. At the end of the $\ell^{th}$ round, if $C(t_{\ell+1}-1,1)<0$, we know $C(t_{\ell+1}-1,1)=-\ell$. 

Let us consider a simple random walk W(t) such that $W(t+1)=W(t)+1$ with probability $p>\frac{1}{2}$ and $W(t+1)=W(t)-1$ with probability $1-p$ for $t\geq 1$. If we denote $p_{\ell}^{*}=P(\exists t_{*}, W(t_{*})=-\ell)$ for $\ell>0$, then it is easy to calculate that $p_{\ell}^{*}=\left(\frac{1-p}{p}\right)^{\ell}$.

Now let us consider $C(t,1)$. If we pull arm $1$ with some other arm $i$ at time t, then $C(t,1)=C(t-1,1)+1$ happens with probability $p_{1,i}>p$ and $C(t,1)=C(t-1,1)-1$ with probability $1-p_{1,i}<1-p$. If we do not pull arm $1$ at time $t$, then $C(t,1)=C(t-1,1)$ with probability $1$.

Define $\tau_{1}=1$ and $\tau_{k}=\min_{t}\{t>\tau_{k-1}, C(t,1)\neq C(\tau_{k-1},1)\}$, for $k=1,2,\cdots, $. Because $\tau_{k}$ is a non-decreasing right continuous stopping time, we know it is a valid random change of time \cite{barndorff2015change}. Define $R(k)$ a new stochastic process where $R(k)=C(\tau_{k},1)$. Then we know at every time k, $R(k)=R(k-1)+1$ with probability greater or equal to p and $R(k)=R(k-1)-1$ with probability less than 1-p. Define $p_{\ell}=P(\exists t_{*},R(t_{*})=-\ell)$, then it is easy to prove $p_{\ell}\leq p_{\ell}^{*}=\left(\frac{1-p}{p}\right)^{\ell}$ using first step analysis and induction (we leave the proof as an exercise for the reader), which means $P(\exists t_{*}, C(t_{*},1)=-\ell)\leq \left(\frac{1-p}{p}\right)^{\ell}$.
\end{proof}

\section{Proof of Lemma~\ref{inequ1}}
\begin{proof}
To show the first claimed equation, we have:
\begin{align}
&\mathbb{E}[B(\ell,k)T_{\ell,k}\bar{D}(\ell)] \nonumber \\
=& \mathbb{E}[B(\ell,k)T_{\ell,k}|\bar{D}(\ell)=1] P(\bar{D}(\ell)=1). \label{eq1}
\end{align}

The first term $\mathbb{E}[B(\ell,k)T_{\ell,k}|\bar{D}(\ell)=1]$ can be bounded by writing it as 
$\mathbb{E}[B(\ell,k)T_{\ell,k}|\bar{D}(\ell)=1]
=  \mathbb{E}[ \mathbb{E}[B(\ell,k)T_{\ell,k}|\bar{D}(\ell)=1, \Arms(\ell,k)] | \bar{D}(\ell)=1]$,
where $\Arms(\ell,k)$ denotes the pair of arms being pulled in iteration $k$ round $\ell$. 

We focus on the inner term $\mathbb{E}[B(\ell,k)T_{\ell,k}|\bar{D}(\ell)=1, \Arms(\ell,k)]$.
$B(\ell,k)$ is observable given $\Arms(\ell,k)$. 
If $B(\ell,k) = 0$ then this inner term is $0$.
If $B(\ell,k)= 1$ then this inner term is 
$\mathbb{E}[T_{\ell,k}|\Arms(\ell,k)]$ (where we note that $T_{\ell,k}$ is conditionally independent of $\bar{D}(\ell)$ given $\Arms(\ell,k)$) and is bounded above by $1/(2p-1)$ by Lemma~\ref{adv}. In both cases, the inner term is bounded above by $1/(2p-1)$, and we have that 
$\mathbb{E}[B(\ell,k)T_{\ell,k}|\bar{D}(\ell)=1] \leq 1/(2p-1)$.

Thus, we have that \eqref{eq1} is bounded above by
\begin{equation*}
\frac{1}{2p-1}P(\bar{D}(\ell)=1) 
\leq \frac{1}{2p-1}\left(\frac{1-p}{p}\right)^{\ell-1},
\end{equation*}
where the final inequality follows from Lemma~\ref{tail} and the fact that $\bar{D}(\ell)=1$ implies $L\geq \ell-1$.

To show the second claimed equation, we use the same proof technique used for the first and get:
\begin{equation*}
\mathbb{E}[B(\ell,k)T_{\ell,k}V(\ell,k)] 
\leq \frac{1}{2p-1}P(V(\ell,k)=1).
\end{equation*}

Now we just need to compute $P(V(\ell,k)=1)$. Given $C(t_{\ell}-1,1)=(N-1)(\ell-1)$ at the beginning of round $\ell$, it loses only if there exists a $t_{0}\geq t_{\ell}$ and $C(1,t_{0})=-\ell$. Using the results from Lemma~\ref{tail}, we know $P(V(\ell,k)=1)\leq \left(\frac{1-p}{p}\right)^{\ell}$.  This completes the proof of the second claimed equation.
\end{proof}

\section{Proof of Lemma~\ref{inequ2}}
\begin{proof}
For the first inequality, we know
\begin{align}
&\mathbb{E}\left[\sum_{k=1}^{N-1}\bar{B}(\ell,k)T_{\ell,k}\bar{D}(\ell)\right] \nonumber \\
=& \sum_{k=1}^{N-1}\mathbb{E}\left[\mathbb{E}[\bar{B}(\ell,k)T_{\ell,k}|D(\ell)=0]\bar{D}(\ell)\right]. \label{eq3}
\end{align}

Moreover,
\begin{align}
&\mathbb{E}[\bar{B}(\ell,k)T_{\ell,k}|D(\ell)=0] \nonumber \\
=&\mathbb{E}[T_{\ell,k}|B(\ell,k)=0,D(\ell)=0]P(B(\ell,k)=0|D(\ell)=0) \nonumber \\
\leq & \frac{N\ell}{2p-1}P(B(\ell,k)=0|D(\ell)=0), \nonumber
\end{align}
where the last equation follows from applying Lemma~\ref{adv} and iterated conditional expectation.
Thus, we know

\begin{align}
(\ref{eq3})=& \sum_{k=1}^{N-1}\frac{N\ell}{2p-1}P(B(\ell,k)=0|D(\ell)=0)\mathbb{E}[\bar{D}(\ell)] \nonumber \\
\leq & \sum_{k=1}^{N-1}\frac{N\ell}{2p-1}P(B(\ell,k)=0|D(\ell)=0)\left(\frac{1-p}{p}\right)^{\ell-1} \label{eq4} \\
\leq & \left(\frac{1-p}{p}\right)^{\ell-1}\frac{2N\ell p^2}{(2p-1)^4}(\log(N)+1), \nonumber
\end{align}
where equation~(\ref{eq4}) is because Lemma~\ref{count}. 

The proof of the second inequality follows very similarly, and is omitted.
\end{proof}
\begin{figure*}[!t]
\begin{tiny}
$
\left[
\begin{array}{*{16}c}
0.5 & 0.512 & 0.622 & 0.655 & 0.698 & 0.726 & 0.711 & 0.708 & 0.749 & 0.8 & 0.741 & 0.783 & 0.847 & 0.817 & 0.854 & 0.868 \\
0.488 & 0.5 & 0.602 & 0.683 & 0.652 & 0.776 & 0.663 & 0.683 & 0.738 & 0.709 & 0.786 & 0.802 & 0.83 & 0.85 & 0.871 & 0.873 \\
0.378 & 0.398 & 0.5 & 0.528 & 0.554 & 0.533 & 0.534 & 0.591 & 0.573 & 0.593 & 0.661 & 0.705 & 0.734 & 0.672 & 0.787 & 0.822 \\
0.345 & 0.317 & 0.472 & 0.5 & 0.553 & 0.619 & 0.566 & 0.641 & 0.675 & 0.687 & 0.665 & 0.696 & 0.803 & 0.823 & 0.796 & 0.844 \\
0.302 & 0.348 & 0.446 & 0.447 & 0.5 & 0.513 & 0.524 & 0.518 & 0.608 & 0.538 & 0.643 & 0.61 & 0.695 & 0.672 & 0.681 & 0.775 \\
0.274 & 0.224 & 0.467 & 0.381 & 0.487 & 0.5 & 0.513 & 0.559 & 0.575 & 0.621 & 0.591 & 0.701 & 0.702 & 0.787 & 0.829 & 0.811 \\
0.289 & 0.337 & 0.466 & 0.434 & 0.476 & 0.487 & 0.5 & 0.559 & 0.553 & 0.613 & 0.564 & 0.607 & 0.703 & 0.735 & 0.736 & 0.801 \\
0.292 & 0.317 & 0.409 & 0.359 & 0.482 & 0.441 & 0.441 & 0.5 & 0.556 & 0.527 & 0.562 & 0.58 & 0.668 & 0.805 & 0.777 & 0.767 \\
0.251 & 0.262 & 0.427 & 0.325 & 0.392 & 0.425 & 0.447 & 0.444 & 0.5 & 0.512 & 0.548 & 0.542 & 0.612 & 0.786 & 0.71 & 0.685 \\
0.2 & 0.291 & 0.407 & 0.313 & 0.462 & 0.379 & 0.387 & 0.473 & 0.488 & 0.5 & 0.543 & 0.579 & 0.613 & 0.718 & 0.685 & 0.747 \\
0.259 & 0.214 & 0.339 & 0.335 & 0.357 & 0.409 & 0.436 & 0.438 & 0.452 & 0.457 & 0.5 & 0.564 & 0.625 & 0.618 & 0.702 & 0.684 \\
0.217 & 0.198 & 0.295 & 0.304 & 0.39 & 0.299 & 0.393 & 0.42 & 0.458 & 0.421 & 0.436 & 0.5 & 0.542 & 0.644 & 0.7 & 0.733 \\
0.153 & 0.17 & 0.266 & 0.197 & 0.305 & 0.298 & 0.297 & 0.332 & 0.388 & 0.387 & 0.375 & 0.458 & 0.5 & 0.577 & 0.607 & 0.596 \\
0.183 & 0.15 & 0.328 & 0.177 & 0.328 & 0.213 & 0.265 & 0.195 & 0.214 & 0.282 & 0.382 & 0.356 & 0.423 & 0.5 & 0.578 & 0.637 \\
0.146 & 0.129 & 0.213 & 0.204 & 0.319 & 0.171 & 0.264 & 0.223 & 0.29 & 0.315 & 0.298 & 0.3 & 0.393 & 0.422 & 0.5 & 0.586 \\
0.132 & 0.127 & 0.178 & 0.156 & 0.225 & 0.189 & 0.199 & 0.233 & 0.315 & 0.253 & 0.316 & 0.267 & 0.404 & 0.363 & 0.414 & 0.5 
\end{array}
\right]
$
\end{tiny}

\caption{User's preference matrix for the Sushi experiment}
\label{Table1}
\end{figure*}
\begin{figure*}[t]
    \centering
    \begin{subfigure}[h!]{0.5\textwidth}
        \centering
        \includegraphics[width=1\textwidth]{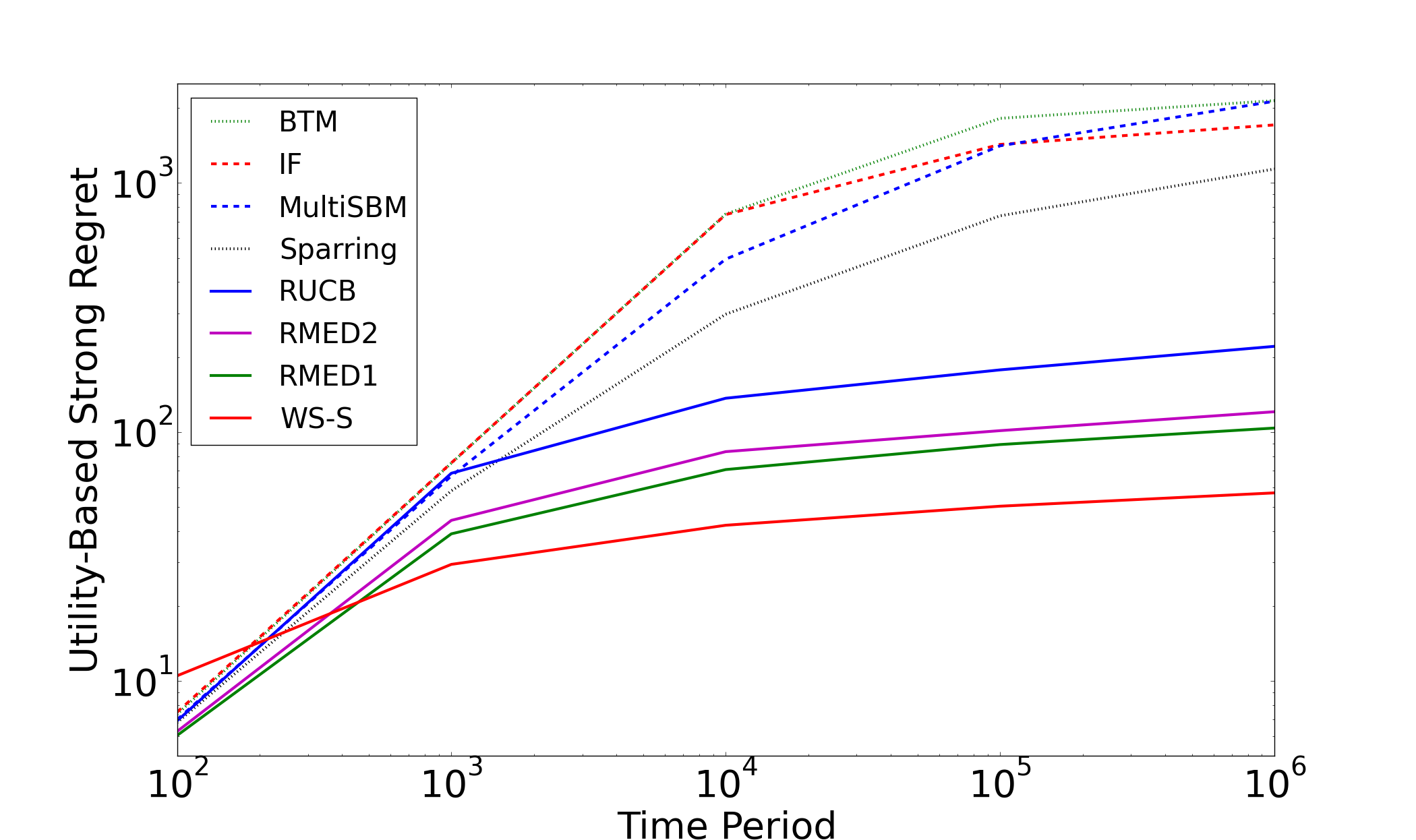}
        \caption{Cyclic dataset with utility-based strong regret}
        \label{fig:sushi2}
    \end{subfigure}%
    ~
    \begin{subfigure}[h!]{0.5\textwidth}
        \centering
        \includegraphics[width=1\textwidth]{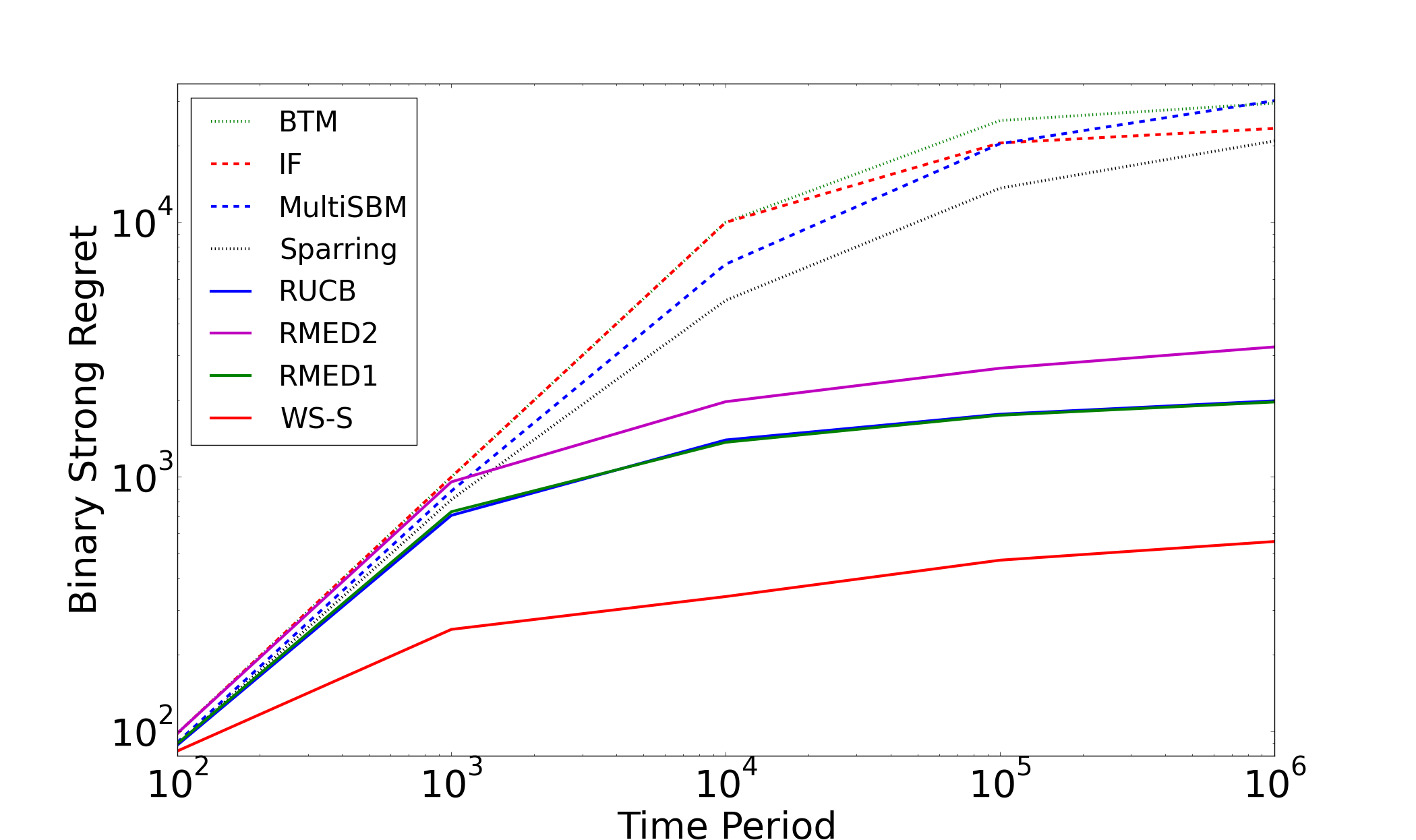}
        \caption{Cyclic dataset with binary strong regret}   
        \label{fig:sushi1} 
    \end{subfigure} 
    \caption{Comparison of the strong regret between \WSSTRONG\  and 7 benchmarks on the cyclic dataset. \WSSTRONG\ outperforms all benchmarks in all settings studied.
\label{fig:result4}}
\end{figure*}

\section{Proof of Theorem 2}
In this section, we prove the cumulative expected weak regret of \WSWEAK\ is bounded by $O(N^2)$ in the Condorcet winner setting. First, we want to give an example to illustrate why our algorithm will not have $O(N\log(N))$ regret under the Condorcet winner setting.

In the Condorcet winner setting, Lemma~\ref{count} is no longer true. Here is a counter example to illustrate why Lemma~\ref{count} does not hold true anymore. Suppose we have $N=3k+1$ arms in total, which includes a Condorcet winner arm and three types of other arms: k type-A arms, k type-B arms and k type-C arms. Among these arms, we assume the user prefers type-A arms than type-B arms, type-B arms than type-C arms and type-C arms than type-A arms. Among each type of arms, there is a total order. In this setting, the expected number of iterations with a worse incumbent is $O(N)$ instead of $O(\log(N))$, which means Lemma~\ref{count} is no longer true.

Now we start our proof for Theorem 2.

\begin{proof}
In the Condorcent winner setting, Lemmas~\ref{tail} and \ref{inequ1} hold, but as explained earlier, Lemma~\ref{count} does not. Because the proof of Lemma~\ref{inequ2} utilizes Lemma~\ref{count}, Lemma~\ref{inequ2} also no longer holds.

On the other hand, since we can have at most $N-1$ iterations in a round, we know the following statement is true: the conditional expected number of iterations with a worse incumbent is bounded by $N$ in each round. Thus, we know Lemma~\ref{inequ2} now becomes:

\begin{align}
&\mathbb{E}\left[\sum_{k=1}^{N-1}\bar{B}(\ell,k)T_{\ell,k}\bar{D}(\ell)\right] \leq \left(\frac{1-p}{p}\right)^{\ell-1}\frac{N^2\ell }{2p-1}, \nonumber \\
&\mathbb{E}\left[\sum_{k=1}^{N-1}\bar{B}(\ell,k)T_{\ell,k}V(\ell,k)\right] \leq \left(\frac{1-p}{p}\right)^{\ell}\frac{N^2\ell }{2p-1}. \nonumber
\end{align}

Thus, following the same reasoning as in the proof of Theorem 1, we know the expected weak regret in the Condorcet winner setting is bounded by
\begin{align}
    \frac{NR}{(2p-1)^2}+\frac{pN^2}{(2p-1)^3}, \nonumber
\end{align}

which concludes our proof.

\end{proof}

\section{Preference Matrices}

In the sushi experiment, the user's preference matrix is given by Figure~\ref{Table1}.

\vspace{1mm}

In the MSLR experiment, the ranker's preference matrix is given by:

\begin{gather*}
\left[
\begin{array}{*{5}c}
0.5 & 0.535 & 0.613 & 0.757 & 0.765 \\
0.465 & 0.5 & 0.580 & 0.727 & 0.738 \\
0.387 & 0.420 & 0.5 & 0.659 & 0.669 \\
0.243 & 0.276 & 0.341 & 0.5 & 0.510 \\
0.235 & 0.262 & 0.331 & 0.490 & 0.5
\end{array}
\right]
\end{gather*}

\section{Condorcet Winner Experiment}

In the main paper, we considered numerical examples in which the arms have a total order.  This is common in the dueling bandits literature, where even work that considers more general settings theoretically test their methods on problems that satisfy the total order assumption \cite{komiyama2016copeland,urvoy2013generic}.

In this section, we consider an additional example that   has a Condorcet winner but does not have a total order among arms. 
The example has a cyclic struture, and is similar to the cyclic example in \citet{komiyama2015regret}. 

The preference matrix is:

\begin{gather*}
\left[\begin{array}{*{4}c}
0.5 & 0.6 & 0.6 & 0.6 \\
0.4 & 0.5 & 0.6 & 0.4 \\
0.4 & 0.4 & 0.5 & 0.6 \\
0.4 & 0.6 & 0.4 & 0.5 
\end{array}\right]
\end{gather*}

In the above example, arm $1$ is the Condorcet winner. Arm $2$ beats arm $3$, arm $3$ beats arm $4$ and arm $4$ beats arm $2$. 

Again, we consider both binary strong regret and the utility-based strong regret. The utility-based strong regret is defined the same as the other two experiments. The result is summarized in Figure~\ref{fig:result4}. \WSSTRONG\ outperforms all benchmarks considered in all time periods on binary regret, and outperforms them all in all time periods except $T=10^2$ on utility-based regret.

\section{Sensitivity Analysis}

In this section, we conduct a sensitivity analysis of $\beta$ in \WSSTRONG\ using the MSLR dataset. In this analysis, we choose $\beta=1.01, 1.05, 1.1, 1.2, 1.5$ respectively and compare them with RMED and RUCB. The result is summarized in Figure~\ref{sap}.

\begin{figure}[h]
\includegraphics[scale=0.15]{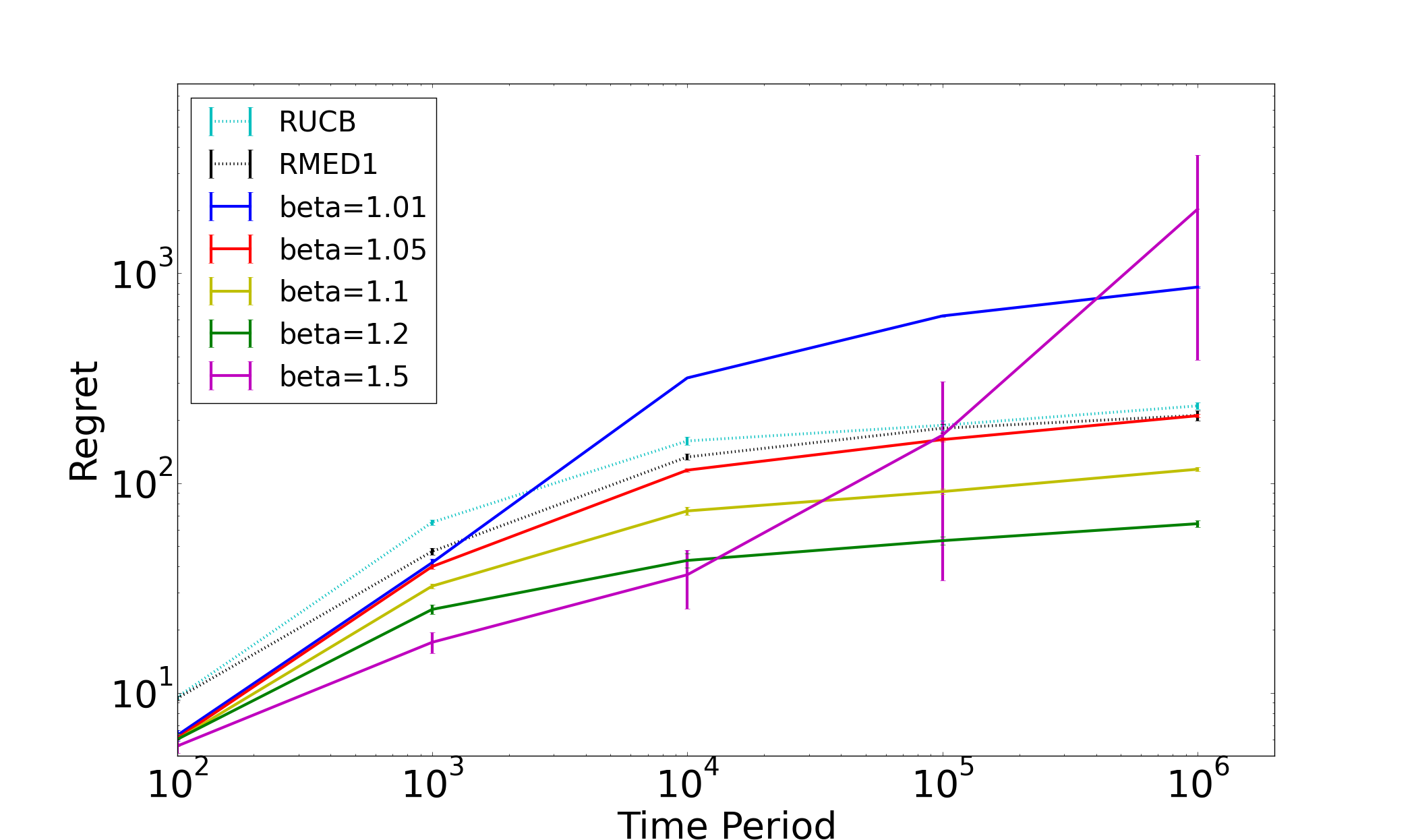}
\caption{Sensitivity Analysis}
\label{sap}
\end{figure}

Based on Figure~\ref{sap}, \WSSTRONG\ with $\beta=1.05, 1.1, 1.2$ outperforms RMED and RUCB. When $\beta=1.01$, we spend too much time on the exploration period and do not exploit enough. Similarly, \WSSTRONG\ with $\beta=1.5$ over exploits and does not explore enough. In both cases, \WSSTRONG\ underperforms RMED and RUCB. However, as long as $\beta$ is within a reasonable range, \WSSTRONG\ can outperform existing state-of-art algorithms.

\end{document}